\documentclass{article} 
\usepackage{iclr2027_conference,times}

\usepackage{amsmath,amsfonts,bm}

\def\eqref#1{equation~\ref{#1}}

\def\1{\bm{1}}

\DeclareMathAlphabet{\mathsfit}{\encodingdefault}{\sfdefault}{m}{sl}
\SetMathAlphabet{\mathsfit}{bold}{\encodingdefault}{\sfdefault}{bx}{n}

\renewcommand{\eqref}[1]{\textup{(\ref{#1})}}

\usepackage{url}
\usepackage{booktabs}
\usepackage{multirow}
\usepackage{graphicx}
\usepackage{tikz}
\usetikzlibrary{arrows.meta}
\usepackage{wrapfig}
\usepackage{needspace}
\usepackage{caption}
\usepackage{amssymb,amsthm}
\usepackage{dsfont}
\newcommand{\ones}{\mathds{1}_V}
\usepackage{algorithm}
\usepackage{algpseudocode}
\usepackage{enumitem}
\usepackage{tcolorbox}
\newtcolorbox{interpretation}{colback=black!3,colframe=black!45,boxrule=0.4pt,arc=1.5pt,
  left=5pt,right=5pt,top=3pt,bottom=3pt,before skip=6pt,after skip=6pt}
\usepackage{titletoc}
\usepackage{hyperref}
\newtheorem{proposition}{Proposition}[section]
\DeclareMathOperator{\NELBO}{NELBO}
\DeclareMathOperator{\ELBO}{ELBO}
\newcommand{\defeq}{\stackrel{\text{def}}{=}}
\makeatletter
\newcommand{\inlinechoice}[2]{\if@display#1\else\mathchoice{#1}{#2}{#1}{#1}\fi}
\makeatother
\newcommand{\Zhat}{\inlinechoice{\widehat Z_0^i}{\smash[t]{\widehat Z_0^i}\vphantom{Z_0^i}}}
\newcommand{\Ztilde}{\inlinechoice{\widetilde Z_0^i}{\smash[t]{\widetilde Z_0^i}\vphantom{Z_0^i}}}

\title{Binary Discrete Diffusion: Candidate-Set\\ Diffusion for Logic and Planning Tasks}

\author{%
\makebox[\dimexpr\textwidth-2\tabcolsep\relax][c]{\normalsize\begin{tabular}[t]{@{}c@{}}
Rodrigo Casado Noguerales\textsuperscript{1,2,$*$}%
\quad Dimitri von Rütte\textsuperscript{1}%
\quad Bernhard Schölkopf\textsuperscript{1,2,3}\\
Thomas Hofmann\textsuperscript{1}%
\quad Aran Raoufi\textsuperscript{1}\\[3pt]
\mdseries\textsuperscript{1}ETH Zurich%
\quad\textsuperscript{2}Max Planck Institute for Intelligent Systems%
\quad\textsuperscript{3}ELLIS Institute Tübingen%
\end{tabular}}%
}

\defcitealias{gemma2026technicalreport}{Gemma Team, 2026}

\iclrfinalcopy 
\begin{document}

\maketitle
\lhead{Preprint} 

\ificlrfinal
{\renewcommand{\thefootnote}{\fnsymbol{footnote}}\footnotetext[1]{Correspondence to: \href{mailto:rcasado@ethz.ch}{\texttt{rcasado@ethz.ch}}.}}
\fi

\begin{abstract}
  Discrete diffusion models iteratively denoise a noisy sequence of categorical data such as text, enabling any-order generation with self-correction and in-filling capabilities.
  This makes them promising candidates for reasoning, planning, and constrained-generation tasks like coding and math.
  Yet masked and uniform diffusion erase a token's identity in a single corruption event, so their noisy states cannot explicitly represent partially resolved uncertainty in which several alternatives remain plausible.
  We close this gap with \textbf{Binary Discrete Diffusion}, a family of continuous-time discrete-diffusion processes whose noisy state is a set of candidate tokens per position, represented as a binary membership vector.
  Binary diffusion changes candidate sets through additions and removals, allowing information to be lost and recovered gradually while the diffusion state itself represents partial uncertainty.
  On Sudoku-Extreme, Boxoban, and graph coloring, the symmetric formulation of binary diffusion consistently outperforms strong, well-tuned masked and uniform baselines under matched architectures and controlled training, sampling, and tuning budgets, with its largest gains on the hardest instances, while remaining competitive on language modeling (Text8, TinyStories, OpenWebText).
  These results support candidate sets as an effective intermediate representation for iterative reasoning and show that the corruption kernel can materially change the computation learned by the reverse model.
\end{abstract}

\begin{figure}[h]
  \centering
\definecolor{kmasked}{HTML}{6699CC}%
\definecolor{kuniform}{HTML}{EE99AA}%
\definecolor{ksymmetric}{HTML}{994455}%
\definecolor{kabsorbing}{HTML}{004488}%
\newcommand{\hcvertices}[3]{\foreach \n/\x/\y in
    {E/0/-0.84, 1/-0.99/-0.085, 3/0.99/-0.085, 2/0/0.17, 12/-0.99/0.925, 23/0.99/0.925, 13/0/0.67, T/0/1.68}
  \coordinate (#1-\n) at ({#2+\x},{#3+\y});}%
\newcommand{\hctokens}[1]{\foreach \v/\t in {1/A, 2/B, 3/C}
  \node[tok, hcv #1-\v/.try] (#1-\v) at (#1-\v) {\t};}%
\newcommand{\hcsets}[1]{%
  \node[cset, hcv #1-T/.try] (#1-T) at (#1-T) {\{A,B,C\}};
  \node[cset, hcv #1-12/.try] (#1-12) at (#1-12) {\{A,B\}};
  \node[cset, hcv #1-23/.try] (#1-23) at (#1-23) {\{B,C\}};
  \node[cset, hcv #1-13/.try] (#1-13) at (#1-13) {\{A,C\}};}%
\newcommand{\hcskel}[1]{\foreach \u/\v in {T/12, T/23, T/13, 12/1, 12/2, 23/2, 23/3, 13/1, 13/3, 1/E, 2/E, 3/E}
  \draw[skel] (#1-\u) -- (#1-\v);}%
\begin{tikzpicture}[
    tok/.style={circle, draw=black!40, line width=0.4pt, fill=white, minimum size=9.5pt,
      inner sep=0pt, font=\scriptsize},
    cset/.style={draw=black!40, line width=0.4pt, fill=white, rounded corners=1pt,
      minimum height=9pt, inner sep=1.1pt, font=\fontsize{6}{7}\selectfont},
    mask/.style={cset, fill=black!10, font=\scriptsize},
    visited/.style={draw=#1, line width=1.2pt},
    skel/.style={black!15, line width=0.4pt},
    move/.style={line width=0.5pt, opacity=0.45, shorten <=0.6pt, shorten >=0.6pt},
    oneway/.style={move, -{Stealth[length=3pt, width=2.8pt]}},
    twoway/.style={move, {Stealth[length=3pt, width=2.8pt]}-{Stealth[length=3pt, width=2.8pt]}},
    walk/.style={line width=1.5pt, -{Stealth[length=5.6pt, width=5.2pt]},
      shorten <=1pt, shorten >=1pt},
    kname/.style={font=\footnotesize\bfseries, anchor=base, text=#1!75!black},
    head/.style={font=\scriptsize\itshape, anchor=base, text=black!60},
  ]
  \tikzset{hcv s-23/.style={visited=ksymmetric}, hcv s-3/.style={visited=ksymmetric},
    hcv s-13/.style={visited=ksymmetric}, hcv s-1/.style={visited=ksymmetric}}
  \hcvertices{s}{0}{0}
  \hcskel{s}
  \hcsets{s}
  \hctokens{s}
  \foreach \u/\v in {T/12, T/23, T/13, 12/1, 12/2, 23/2, 23/3, 13/1, 13/3}
    \draw[twoway, ksymmetric] (s-\u) -- (s-\v);
  \draw[walk, ksymmetric] (s-23) -- (s-3);
  \draw[walk, ksymmetric] (s-3) -- (s-13);
  \draw[walk, ksymmetric] (s-13) -- (s-1);
  \node[kname=ksymmetric] at (0,2.62) {Binary symmetric};
  \node[head] at (0,2.25) {candidate set, non-absorbing};

  \tikzset{hcv u-1/.style={visited=kuniform}, hcv u-2/.style={visited=kuniform},
    hcv u-3/.style={visited=kuniform}}
  \hcvertices{u}{3.7}{0}
  \hcskel{u}
  \hctokens{u}
  \foreach \u/\v in {1/2, 2/3, 3/1} \draw[twoway, kuniform] (u-\u) -- (u-\v);
  \draw[walk, kuniform] (u-3) -- (u-2);
  \draw[walk, kuniform] (u-2) -- (u-1);
  \node[kname=kuniform] at (3.7,2.62) {Uniform};
  \node[head] at (3.7,2.25) {single token, non-absorbing};

  \tikzset{hcv a-T/.style={visited=kabsorbing}, hcv a-13/.style={visited=kabsorbing},
    hcv a-1/.style={visited=kabsorbing}}
  \hcvertices{a}{7.4}{0}
  \hcskel{a}
  \hcsets{a}
  \hctokens{a}
  \foreach \u/\v in {T/12, T/23, T/13, 12/1, 12/2, 23/2, 23/3, 13/1, 13/3}
    \draw[oneway, kabsorbing] (a-\u) -- (a-\v);
  \draw[walk, kabsorbing] (a-T) -- (a-13);
  \draw[walk, kabsorbing] (a-13) -- (a-1);
  \node[kname=kabsorbing] at (7.4,2.62) {Binary absorbing};
  \node[head] at (7.4,2.25) {candidate set, absorbing};

  \tikzset{hcv m-T/.style={visited=kmasked}, hcv m-1/.style={visited=kmasked}}
  \hcvertices{m}{11.1}{0}
  \hcskel{m}
  \node[mask, hcv m-T/.try] (m-T) at (m-T) {$\mathtt m$};
  \hctokens{m}
  \foreach \v in {1,2,3} \draw[oneway, kmasked] (m-T) -- (m-\v);
  \draw[walk, kmasked] (m-T) -- (m-1);
  \node[kname=kmasked] at (11.1,2.62) {Masked};
  \node[head] at (11.1,2.25) {single token, absorbing};
\end{tikzpicture}%
  \caption{Denoising process of each kernel over the candidate sets of a three-token vocabulary \mbox{\{A,B,C\}}, with its allowed moves faded. Each panel highlights one possible denoising path to A.}
  \label{fig:hypercube}
\end{figure}
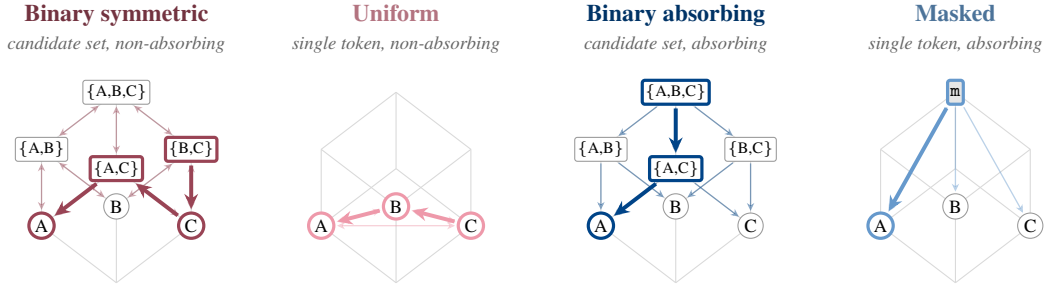

\section{Introduction}

Discrete diffusion models generate by iteratively refining a corrupted sequence.
Unlike autoregressive generation, decisions need not follow a fixed left-to-right order and can be revised later, while many positions can be updated in parallel~\citep{schiff2025udlm,sahoo2024mdlm,gong2025diffullama}.
These properties make discrete diffusion attractive for reasoning, planning, and constrained generation~\citep{ye2024beyondautoregression,wewer2025spatialreasoning}.
Yet sample quality at finite inference budgets remains challenging, motivating work on ordering, remasking, planned denoising, and search~\citep{kim2025tokenordering,wang2025remaskingdiscretediffusion,liu2025planneddenoising,bai2026prism}.
We instead focus on a complementary axis: the corruption process itself.
How much can be gained by changing the corruption process, and in particular the intermediate representation on which denoising operates? Logic and planning are useful testbeds: they impose global constraints, admit exact solve-rate evaluation, and scale cleanly in difficulty at manageable computational cost, while sharing structure with coding, mathematics, and biological design~\citep{sahoo2026scaling,avdeyev2023dirichlet,stark2024dirichlet}.
They also expose a mismatch in the dominant masked and uniform kernels~\citep{sahoo2024mdlm,shi2024simplified,schiff2025udlm,sahoo2025diffusionduality}: a position cannot naturally express that several alternatives remain plausible.
A model solving Sudoku, for example, might narrow a cell from $\{2,7,9\}$ to $\{7,9\}$, later reopen $2$ as neighboring decisions are revised, and eventually settle on $7$.

We propose to make such candidate sets themselves the diffusion states.
A clean token is a singleton, while during reverse generation intermediate states can carry several live hypotheses that may be eliminated or restored.
More broadly, the corruption kernel does more than specify how examples are noised: it determines what information survives at intermediate times and how the model's prediction is translated into reverse moves and their rates.
Identical architectures can therefore induce qualitatively different iterative computations under different kernels.
We instantiate this idea with \textbf{Binary Discrete Diffusion}, representing candidate sets by binary membership vectors in $\{0,1\}^V$ (Figure~\ref{fig:hypercube}).
In \textbf{binary symmetric diffusion}, membership bits can switch both on and off, allowing hypotheses to be eliminated and later restored.
In \textbf{binary absorbing diffusion}, corruption only adds candidates, so reverse generation progressively prunes the set toward a singleton.

Although each position has a nominal state space of size $2^V$, every state has only $V$ single-bit neighbors, and all neighboring reverse rates can be computed in $O(V)$ time from the model's usual $V$-way prediction, without enumerating subsets.
The model learns these dynamics from clean solutions without candidate labels, task-specific deduction rules, or symbolic search.
The richer state also enables \textbf{candidate-aware sampling}: we derive exact ancestral updates through a shared clean-token proposal, adapt re-corrupt sampling, and, building on $\Psi$-sampling and confidence-guided remasking~\citep{deschenaux2026psisamplers,wang2025remaskingdiscretediffusion}, develop revision rules that choose both \emph{which positions} to reconsider and \emph{which hypotheses} to reopen.

We evaluate Binary Discrete Diffusion on Sudoku-Extreme~\citep{wang2025hrm}, Boxoban~\citep{guez2018boxoban}, and graph coloring under controlled architectures, training budgets, and sampler tuning, using direct single-attempt solutions without verifier-guided selection.
Across all three tasks, a consistent pattern emerges: \textbf{binary symmetric diffusion outperforms uniform diffusion, and both outperform masked and binary absorbing diffusion}.
The advantage of binary symmetric over uniform grows on harder regimes, including low-clue Sudoku, long-horizon Boxoban, and graph coloring far beyond the training size.
On Text8, TinyStories, and OpenWebText~\citep{mahoney2011text8,eldan2023tinystories,gokaslan2019openwebtext}, binary diffusion remains competitive with established kernels.
These results support candidate sets as an effective intermediate representation for iterative reasoning and show that the corruption process can materially change the computation learned by the reverse model.

\section{Discrete diffusion modeling}
\label{sec:discrete}

We consider a data distribution $p_{\mathrm{data}}$ on sequences of length $L$ over a finite vocabulary $[V]=\{1,\ldots,V\}$.
Starting from $Z_0\sim p_{\mathrm{data}}$, we corrupt its positions independently with a continuous-time Markov chain (CTMC) $Z_t=(Z_t^1,\ldots,Z_t^L)$, whose distribution approaches an easily sampled noise distribution as $t\uparrow1$.
We write $q_{t\mid s}(y\mid z)$ for the one-position transition distribution or \emph{kernel} between times $s<t$ and states $z,y$.
For masked and uniform diffusion we use~\citep{austin2021d3pm}
\begin{equation}
  q_{t\mid0}(y\mid c)=
  \begin{cases}
    (1-t)\mathbf1\{y=c\}+t\mathbf1\{y=\mathtt m\}, & \text{masked},  \\
    (1-t)\mathbf1\{y=c\}+t/V,                      & \text{uniform}.
  \end{cases}
  \label{eq:standard-kernels}
\end{equation}
Thus $Z_t$ can be sampled directly at any time $t$, without simulating the intervening jumps.

We write $Q_t(z,y)$ for the forward jump rate from $z$ to $y$, so that
$q_{t+h\mid t}(y\mid z)=\mathbf1\{y=z\}+hQ_t(z,y)+o(h)$ as $h\downarrow0$.
The reverse rates depend on $p_{\mathrm{data}}$, but conditioning on the clean sample $Z_0$ makes them computable from the forward process~\citep{campbell2022continuous}.
For a change at position $i$ from noisy token $Z_t^i$ to another $y\ne Z_t^i$, the clean-conditioned reverse rates are given by
\begin{equation}
  a_i(Z_0,Z_t,y;t)
  \defeq Q_t(y,Z_t^i)
  \frac{q_{t\mid0}(y\mid Z_0^i)}
  {q_{t\mid0}(Z_t^i\mid Z_0^i)}.
  \label{eq:discrete-conditioned-rate}
\end{equation}
The true reverse rates are obtained by averaging these clean-conditioned rates over $q(Z_0\mid Z_t)$.
Let $a_i^\theta(Z_t,y;t)$ denote learned reverse rates that depend on the noisy sequence $Z_t$ and time $t$ but not on $Z_0$.
We train them against the clean-conditioned targets by minimizing the negative of the continuous-time evidence lower bound (ELBO)~\citep{campbell2022continuous,casado2026discretediffusionlearn}:
\begin{equation}
  \begin{gathered}
    -\ELBO(\theta)
    =\int_0^1
    \mathbb E_{Z_0,Z_t}
      [\mathcal J_t(\theta;Z_0,Z_t)]\,dt,
    \\[2pt]
    \mathcal J_t(\theta;Z_0,Z_t)
    =\sum_{i=1}^L\sum_{y\ne Z_t^i}
    \Phi(a_i,a_i^\theta),
    \qquad
    \Phi(a,b)=a\log\frac ab-a+b.
  \end{gathered}
  \label{eq:binary-elbo}
\end{equation}
This upper-bounds the expected negative log-likelihood.
More sharply, its excess above the data-entropy floor is exactly the path-space KL between the exact reverse process and the learned reverse process~\citep{casado2026discretediffusionlearn}.
In practice, the jump rates can become singular near $t=0$ or $t=1$, so we train on an interior window $0<t_1<t_2<1$.
The corresponding finite-window objective includes reconstruction and terminal-prior boundary terms, given in App.~\ref{app:finite-window-nelbo}.

\paragraph{Training.}
We parameterize the model reverse rates through a network that takes $(Z_t,t)$ as input and outputs, at each position $i$, a strictly positive distribution $\mu_i^\theta(c\mid Z_t;t)$ over clean tokens \mbox{$c\in[V]$}.
This induces a noisy-state distribution $m_i^\theta$, which determines the model reverse rates for $y\ne Z_t^i$:
\begin{equation}
  a_i^\theta(Z_t,y;t)
  \defeq Q_t(y,Z_t^i)
  \frac{m_i^\theta(y\mid Z_t;t)}
  {m_i^\theta(Z_t^i\mid Z_t;t)},
  \qquad
  m_i^\theta(y\mid Z_t;t)\defeq\sum_{c=1}^V\mu_i^\theta(c\mid Z_t;t)\,q_{t\mid0}(y\mid c).
  \label{eq:binary-rate-parameterization}
\end{equation}
Thus one $V$-way network output determines all allowed reverse moves at a position.
Reweighting $\mu_i^\theta$ by the kernel $q_{t\mid0}(Z_t^i\mid c)$ gives a new distribution which we denote by
\begin{equation}
  \pi_i^\theta(c\mid Z_t;t)
  \defeq\frac{\mu_i^\theta(c\mid Z_t;t)\,q_{t\mid0}(Z_t^i\mid c)}
  {m_i^\theta(Z_t^i\mid Z_t;t)}.
  \label{eq:binary-clean-distribution}
\end{equation}
In terms of this distribution, the model rate $a_i^\theta(Z_t,y;t)$ from~\eqref{eq:binary-rate-parameterization} becomes an average of the clean-conditioned rate $a_i(Z_0,Z_t,y;t)$ from~\eqref{eq:discrete-conditioned-rate}, where tokens with $\pi_i^\theta(c\mid Z_t;t)=0$ contribute zero:
\begin{equation}
  a_i^\theta(Z_t,y;t)
  =Q_t(y,Z_t^i)\sum_{c=1}^V\pi_i^\theta(c\mid Z_t;t)\,\frac{q_{t\mid0}(y\mid c)}{q_{t\mid0}(Z_t^i\mid c)}
  =\mathbb E_{Z_0^i\sim\pi_i^\theta(\cdot\mid Z_t;t)}\big[a_i(Z_0,Z_t,y;t)\big].
  \label{eq:binary-denoiser-average}
\end{equation}
Training then minimizes an empirical estimate of the path integral in~\eqref{eq:binary-elbo} by sampling $(t,Z_0,Z_t)$.

\refstepcounter{proposition}
\paragraph{Remark \theproposition .}
\label{rem:cavity}
This is the \emph{cavity parameterization}~\citep{casado2026discretediffusionlearn}.
Since the true reverse rates take the averaged form~\eqref{eq:binary-denoiser-average} with the posterior $q(Z_0^i\mid Z_t)$ in place of $\pi_i^\theta$, the model rates equal the true rates, and thus minimize the negative ELBO, when $\pi_i^\theta(c\mid Z_t;t)=q(Z_0^i=c\mid Z_t)$.
This is attained by setting $\mu_i^\theta(c\mid Z_t;t)=q(Z_0^i=c\mid Z_t^{-i})$, the distribution of the clean token given the rest of the sequence $Z_t^{-i}$ (that is, $Z_t$ excluding $Z_t^i$), even though the network is evaluated on the full sequence $Z_t$.
We call $\pi_i^\theta$ the \emph{denoiser} and $\mu_i^\theta$ the \emph{cavity} predictions or distributions.

\section{Binary discrete diffusion}
\label{sec:binary}

In binary diffusion, the noisy state at each position is a \emph{set of candidate tokens}.
We represent a clean token $c\in[V]$ by the singleton $\{c\}$, and corruption changes this set by adding or removing candidates, so that an intermediate state can retain several alternatives at a single position.
Representing each set by its binary membership vector identifies the noisy state $Z_t^i$ with a vertex of the binary cube $\{0,1\}^V$, and we use the same symbol for a noisy set $z$ and its membership vector, including for a clean singleton $\{c\}$ and its one-hot vector.
We consider two corruption processes on this space: \textbf{binary symmetric diffusion}, under which candidates can be both removed and restored, and \textbf{binary absorbing diffusion}, which only adds candidates.
We focus primarily on the symmetric process and treat the absorbing process as a simpler monotone counterpart.

Although the bits evolve independently when conditioned on the clean sequence $Z_0$, the true and learned reverse dynamics are generally correlated across positions and depend on the full noisy sequence.
Crucially, though, the network still produces only the usual $V$-way distribution $\mu_i^\theta(c\mid Z_t;t)$ over clean tokens at each position, and the construction of Section~\ref{sec:discrete} converts this single prediction into all allowed reverse rates on the binary cube.
This is what makes candidate-set diffusion tractable, as we explain in more detail in Section~\ref{sec:binary-embedding}.

\subsection{Binary symmetric diffusion}
\label{sec:binary-symmetric}

In binary symmetric diffusion, each membership bit retains its clean value with probability $1-t$ and is otherwise resampled uniformly from $\{0,1\}$.
A resampled bit keeps its value half of the time, so a bit differs from its clean value with probability $t/2$, and given $Z_0$ the forward process factorizes over positions and bits with the one-position kernel
\begin{equation}
  q_{t\mid0}^{\mathrm{sym}}(z\mid c)
  =\Big(1-\frac t2\Big)^{V-d_H(z,c)}\Big(\frac t2\Big)^{d_H(z,c)},
  \label{eq:binary-symmetric-clean-kernel}
\end{equation}
where $d_H$ is the Hamming distance.
A state at time $t$ thus contains the clean token $c$ with probability $1-t/2$ and every other token independently with probability $t/2$, so it holds $(V-1)t/2$ false candidates on average, and it approaches independent fair bits, a uniformly random subset of the vocabulary, as $t\uparrow1$.
More generally, between times $0\le s<t<1$ each bit is resampled with probability $(t-s)/(1-s)$, which gives the transition kernel
\begin{equation}
  q_{t\mid s}^{\mathrm{sym}}(y\mid z)
  =(1-\varepsilon_{s,t})^{V-d_H(y,z)}\,\varepsilon_{s,t}^{d_H(y,z)},
  \qquad
  \varepsilon_{s,t}\defeq\frac{t-s}{2(1-s)}.
  \label{eq:binary-symmetric-kernel}
\end{equation}

\paragraph{Forward rates.}
Each bit is resampled at rate $1/(1-t)$ and changes its value half of the time, so differentiating~\eqref{eq:binary-symmetric-kernel} gives the flip rate
\begin{equation}
  Q_t^{\mathrm{sym}}\big(z,z^{(u)}\big)=r_t\defeq\frac1{2(1-t)},
  \label{eq:binary-symmetric-forward-rates}
\end{equation}
where $z^{(u)}$ is the state $z$ with a single bit $u$ flipped, and every other jump has rate zero.
The forward process thus moves along the edges of the binary cube: from $z$, it can only reach the $V$ Hamming neighbors $z^{(1)},\ldots,z^{(V)}$ of $z$ in $\{0,1\}^V$.

\paragraph{Clean-conditioned reverse rates.}
The reverse process traverses the same edges, so each move flips one bit.
We write $a_{i,u}^{\mathrm{sym}}$ for the rate of flipping bit $u$ of $Z_t^i$.
Since the forward kernel~\eqref{eq:binary-symmetric-clean-kernel} factorizes over bits, all unchanged-bit factors cancel in the kernel ratio in~\eqref{eq:discrete-conditioned-rate}.
The remaining ratio is $\rho_t=(1-t/2)/(t/2)=(2-t)/t$ when the flip repairs a mismatch with the clean bit $(Z_0^i)_u$, and $\rho_t^{-1}$ when it creates one.
With the forward rate $r_t$ introduced above, the clean-conditioned rate is therefore
\begin{equation}
  a_{i,u}^{\mathrm{sym}}(Z_0,Z_t;t)
  =r_t\big[
    \rho_t\,\mathbf 1\{(Z_t^i)_u\ne(Z_0^i)_u\}
    +\rho_t^{-1}\,\mathbf 1\{(Z_t^i)_u=(Z_0^i)_u\}
    \big].
  \label{eq:binary-symmetric-odds}
\end{equation}
\textbf{A repairing flip is therefore favored over a corrupting flip by a factor $\rho_t^2$.}
This preference strengthens as $t\downarrow0$ and vanishes as $t\uparrow1$, while both directions remain possible for $0<t<1$.

\paragraph{Model reverse rates.}
Averaging over the denoiser as in~\eqref{eq:binary-denoiser-average} replaces each indicator in~\eqref{eq:binary-symmetric-odds} by its probability, so that, with $\mathbb P_i^\theta$ the probability under $Z_0^i\sim\pi_i^\theta(\cdot\mid Z_t;t)$, the model reverse rates are
\begin{equation}
  a_{i,u}^{\theta,\mathrm{sym}}(Z_t;t)
  =r_t\big[\rho_t\,\mathbb P_i^\theta\{(Z_t^i)_u\ne(Z_0^i)_u\}+\rho_t^{-1}\,\mathbb P_i^\theta\{(Z_t^i)_u=(Z_0^i)_u\}\big].
  \label{eq:binary-symmetric-rates}
\end{equation}
The bit of an active candidate $u$ is wrong unless $u$ is the clean token, and the bit of an inactive token is wrong only if it is the clean token, so removing and restoring a candidate have model reverse rates
\begin{equation}
  a_{i,u}^{\theta,\mathrm{sym}}(Z_t;t)
  =
  \begin{cases}
    r_t\big[\rho_t\,(1-\pi_i^\theta(u\mid Z_t;t))+\rho_t^{-1}\,\pi_i^\theta(u\mid Z_t;t)\big], & u\in Z_t^i,    \\
    r_t\big[\rho_t\,\pi_i^\theta(u\mid Z_t;t)+\rho_t^{-1}\,(1-\pi_i^\theta(u\mid Z_t;t))\big], & u\notin Z_t^i,
  \end{cases}
  \label{eq:binary-symmetric-removal}
\end{equation}
where the denoiser $\pi_i^\theta(\cdot\mid Z_t;t)$ is the cavity $\mu_i^\theta(\cdot\mid Z_t;t)$ multiplied by $\rho_t^2$ on the active candidates of $Z_t^i$ and renormalized.

\begin{interpretation}
  \textbf{Interpretation.}
  Each rate lies between the slow rate $r_t/\rho_t$ and the fast rate $r_t\rho_t$ of~\eqref{eq:binary-symmetric-odds}.
  \textbf{Candidates assigned higher probability by the denoiser are harder to remove when present and faster to restore when absent.}
\end{interpretation}

\subsection{Binary absorbing diffusion}
\label{sec:binary-absorbing}

Binary absorbing diffusion only adds candidates: between times $0\le s<t<1$, each inactive bit activates independently with probability $(t-s)/(1-s)$, and active bits never flip.
Its kernel and forward rates are therefore
\begin{equation}
  q_{t\mid s}^{\mathrm{abs}}(y\mid z)
  =\mathbf 1\{z\subseteq y\}
  \Big(\frac{1-t}{1-s}\Big)^{V-|y|}
  \Big(\frac{t-s}{1-s}\Big)^{|y|-|z|},
  \qquad
  Q_t^{\mathrm{abs}}\big(z,z^{(u)}\big)=\frac{\mathbf 1\{u\notin z\}}{1-t},
  \label{eq:binary-absorbing-transition}
\end{equation}
and every other jump has rate zero.
In particular, $q_{t\mid0}^{\mathrm{abs}}(z\mid c)=\mathbf 1\{c\in z\}(1-t)^{V-|z|}t^{|z|-1}$: the clean token $c$ always remains a candidate and every other token becomes one independently with probability $t$, so the state holds $(V-1)t$ false candidates on average and approaches the full vocabulary as $t\uparrow1$.

Since the forward process only adds candidates, the reverse process only removes them.
Removing a false candidate gives the kernel ratio $(1-t)/t$ in the clean-conditioned rate~\eqref{eq:discrete-conditioned-rate}, while the clean token cannot be removed, so this rate and its average~\eqref{eq:binary-denoiser-average} over the denoiser $\pi_i^\theta$ are
\begin{equation}
  a_{i,u}^{\mathrm{abs}}(Z_0,Z_t;t)
  =\frac{\mathbf 1\{u\in Z_t^i\setminus Z_0^i\}}{t},
  \qquad
  a_{i,u}^{\theta,\mathrm{abs}}(Z_t;t)
  =\frac{\mathbf 1\{u\in Z_t^i\}}{t}\big(1-\pi_i^\theta(u\mid Z_t;t)\big),
  \label{eq:binary-absorbing-rate}
\end{equation}
where $\pi_i^\theta(\cdot\mid Z_t;t)$ is the renormalized restriction of $\mu_i^\theta(\cdot\mid Z_t;t)$ to the active candidates of $Z_t^i$.

\begin{interpretation}
  \textbf{Interpretation.}
  A candidate receiving a larger share of the model probability within the current candidate set has a smaller removal rate.
  \textbf{Denoising therefore prunes unlikely candidates while protecting plausible ones.}
\end{interpretation}

\subsection{Efficient computation and embedding with candidate sets}
\label{sec:binary-embedding}

Candidate-set diffusion has $2^V$ possible states per position, which a priori require work exponential in $V$ for computing the model rates and the ELBO loss.
However, parameterizing in terms of a clean-data distribution allows us to work in \(O(V)\) per position, without enumerating the exponentially many possible candidate sets.
For the rates, every reverse jump flips a single bit, so the reverse process can only move from a state $z$ to its $V$ Hamming neighbors $z^{(u)}$.
After one network evaluation, the denoiser $\pi_i^\theta$ of a position follows from $\mu_i^\theta$ by reweighting and normalization in $O(V)$ time, and each of the $V$ flip rates $a_{i,u}^\theta$,~\eqref{eq:binary-symmetric-removal} or~\eqref{eq:binary-absorbing-rate}, then takes constant work.
For the same reason, the path loss in~\eqref{eq:binary-elbo} sums only over these $V$ moves per position $i$, $\mathcal J_t=\smash[t]{\sum_i\sum_{u=1}^V\Phi(a_{i,u},a_{i,u}^\theta)}$. App.~\ref{app:binary} derives all rates and losses in full detail.

Similarly, to avoid materializing an embedding table over the $2^V$ states, we embed a candidate set by applying the input embedding matrix $E\in\mathbb R^{V\times d}$, whose row $E_c$ embeds clean token $c\in [V]$, to its $\ell_2$-normalized membership vector,
\begin{equation}
  e(z)=E^\top\frac{z}{\|z\|_2}=\frac{1}{\sqrt{|z|}}\sum_{c\in z}E_c,
  \qquad e(\emptyset)=0.
  \label{eq:binary-embedding}
\end{equation}
This costs $O(V)$ and reduces to the ordinary token embedding $E_c$ for a singleton $\{c\}$.
Apart from this input layer, we therefore use the same architecture as for the masked and uniform baselines.

\subsection{Information loss and the time schedule}
\label{sec:binary-schedule}

\begin{figure}[t]
  \centering
  \includegraphics[width=\linewidth]{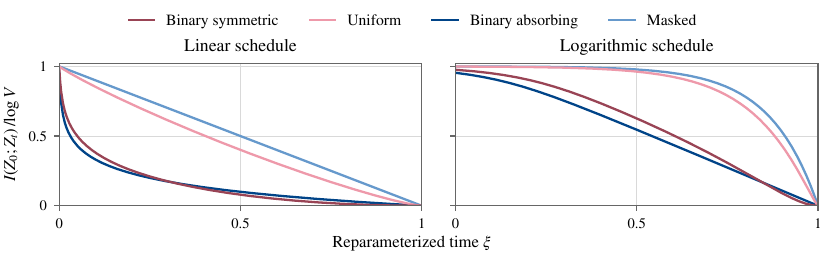}
  \caption{Fraction of the information retained for independent, uniformly distributed tokens with $V=1024$, against the reparameterized time $\xi$ of the linear (left) and logarithmic (right) schedules.}
  \label{fig:information}
\end{figure}

In binary diffusion, the first false candidate doubles the number of candidates for the clean token, whereas each later one enlarges an already large set.
Since false candidates appear at a rate proportional to $V$, the binary kernels thus erase the clean token much faster than masked and uniform diffusion at large vocabularies.

We quantify the retained information by the mutual information
$I(Z_0^i;Z_t^i)=H(Z_0^i)-H(Z_0^i\mid Z_t^i)$,
the reduction in uncertainty about the clean token from observing its noisy state, where $H$ denotes Shannon entropy. Since $I(Z_0^i;Z_t^i)$ is intractable in general, we do so on a reference distribution of independent, uniformly distributed tokens (App.~\ref{app:information}).
Masked and uniform diffusion then lose information roughly linearly in $t$, whereas the binary kernels lose it roughly logarithmically, so that a uniform sample of $t$ places most training times where little information is left (Figure~\ref{fig:information}, left).

This reference calculation motivates a change of clock~\citep{casado2026discretediffusionlearn}:  a time reparameterization $\xi\mapsto t$, with $\xi$ drawn uniformly from $[0,1]$, to control how fast information is destroyed.
Reweighting the path integrand $\mathcal J_t$ by $dt/d\xi$ then leaves the ELBO objective~\eqref{eq:binary-elbo} unchanged, so that only the allocation of training times changes.
With a logarithmic map instead of the linear one,
\begin{equation}
  t_{\mathrm{lin}}(\xi)=t_1+(t_2-t_1)\,\xi,
  \qquad
  t_{\log}(\xi)=t_1\Big(\frac{t_2}{t_1}\Big)^{\xi},
  \label{eq:binary-schedules}
\end{equation}
the binary kernels lose their information approximately linearly in $\xi$ (Figure~\ref{fig:information}, right).
We also give $\xi$ to the network as its time input rather than $t$.
For $V\ge1024$, the logarithmic schedule substantially improved optimization.
For the smaller vocabularies ($V\le 9$) of our logic tasks, this choice made no difference and all kernels trained successfully with the linear schedule.
We further refine all time schedules with an adaptive time sampler (see Apps.~\ref{app:binary-clocks} and~\ref{app:adaptive-time-sampling} for the details).

\section{Sampling}
\label{sec:binary-sampling}

Generation runs the learned reverse process on a grid $t_2=\tau_N>\cdots>\tau_0=t_1$, starting from a prior distribution $p_{t_2}$ and ending with a final decoding step via the denoiser $\pi_i^\theta(\cdot\mid Z_{t_1};t_1)$.
As priors we take the limiting forward distributions, a uniformly random subset of the vocabulary for binary symmetric and the full vocabulary, given by the all-ones vector $\ones$, for binary absorbing:
\begin{equation}
  p_{t_2}^{\mathrm{sym}}(z)\defeq2^{-V},
  \qquad
  p_{t_2}^{\mathrm{abs}}(z)\defeq\mathbf 1\{z=\ones\}.
  \label{eq:binary-priors}
\end{equation}

\paragraph{Ancestral sampling.}
Ancestral sampling from a time $t$ to the next time $s<t$ of the grid follows the reverse process learned through the ELBO, so that as $t\downarrow s$, its jump rates approach the model rates~\eqref{eq:binary-denoiser-average}.
At each step, it updates the positions independently given $Z_t$ by averaging the \emph{bridge} $q_{s\mid t,0}(y\mid z,c)$, the distribution of $Z_s^i=y$ given $Z_t^i=z$ and $Z_0^i=c$, over the denoiser:
\begin{equation}
  p_{s\mid t,i}^\theta(y\mid Z_t)
  =\sum_{c=1}^V\pi_i^\theta(c\mid Z_t;t)\,q_{s\mid t,0}(y\mid Z_t^i,c).
  \label{eq:binary-ancestral}
\end{equation}
As in Section~\ref{sec:binary-embedding}, we avoid enumerating the $2^V$ states: given the clean token, the forward process corrupts the bits independently, so the bridge factors over bits,
\begin{equation}
  q_{s\mid t,0}(y\mid z,c)=\prod_{u=1}^V q_{s\mid t,0}\big(y_u\mid z_u,\mathbf 1\{u=c\}\big).
  \label{eq:binary-bit-bridge}
\end{equation}
Thus, each bit's factor accounts for every sequence of jumps between $s$ and $t$, and a step can change many bits without simulating jumps or Poisson counts.
The mixture~\eqref{eq:binary-ancestral} does not factor over bits, but we can still sample it exactly at a cost of $O(V)$ per position without enumerating the $2^V$ possible states: we draw a clean proposal $\Zhat{}\sim\pi_i^\theta(\cdot\mid Z_t;t)$ and then sample each bit independently from~\eqref{eq:binary-bit-bridge} with $c=\Zhat{}$ (App.~\ref{app:binary-sampling}).
Under the symmetric bridge, each bit favors its current value for $s$ close to $t$ and the proposal's value as $s\downarrow0$, so removed candidates can return, whereas the absorbing bridge always retains the clean proposal and independently retains each other active candidate with probability $s/t$.

\paragraph{Revision by re-corruption.}
Ancestral sampling conditions each step on the current state $Z_t$, which biases it toward what earlier steps have decided, including their mistakes.
\emph{Re-corrupt} sampling instead forgets the current state: it averages the kernel $q_{s\mid0}$ rather than the bridge $q_{s\mid t,0}$,
\begin{equation}
  p_{s\mid t,i}^{\theta,\mathrm{re\text{-}corr}}(y\mid Z_t)
  =\sum_{c=1}^V\pi_i^\theta(c\mid Z_t;t)\,q_{s\mid0}(y\mid c),
  \label{eq:binary-recorrupt}
\end{equation}
which we again sample in two steps by drawing a clean proposal $\Ztilde{}\sim\pi_i^\theta(\cdot\mid Z_t;t)$ and then corrupting it afresh to noise level $s$.
With an exact denoiser, both updates preserve the per-position forward marginals, but re-corruption makes every position revisable at every step.
Interestingly, re-corruption alone markedly improves solve rates on logic tasks (Table~\ref{tab:logic-sampling}, App.~\ref{app:sampler-tuning}).

\paragraph{Confidence-guided $\Psi$-sampler.}
As in the $\Psi$-sampler of~\citet{deschenaux2026psisamplers}, we mix the ancestral and re-corrupt updates, but independently at each position $i$ with weight $\kappa_i$:
\begin{equation}
  \Zhat{},\Ztilde{}\sim\pi_i^\theta(\cdot\mid Z_t;t),\qquad
  Z_s^i\sim(1-\kappa_i)\,q_{s\mid t,0}(\cdot\mid Z_t^i,\Zhat{})
  +\kappa_i\,q_{s\mid0}(\cdot\mid\Ztilde{}),
  \label{eq:binary-psi-sampler}
\end{equation}
where we draw the two proposals independently and where $\kappa_i$ may depend on the re-corrupted proposal $\Ztilde{}$.
A constant $\kappa_i=\eta$ interpolates from ancestral sampling at $\eta=0$ to re-corrupt sampling at $\eta=1$ and preserves the per-position marginals, since both updates do.
For masked diffusion, this mixture is the remasking sampler ReMDM~\citep{wang2025remaskingdiscretediffusion}.
We extend it with two confidence mechanisms, position confidence with strengths $\gamma$ and $\lambda$ and candidate confidence with strength $\beta$, which recover the plain $\Psi$-sampler at zero (Algorithm~\ref{alg:binary-sampling}; full details specified in App.~\ref{app:psi-sampler}).
\begin{itemize}[nosep,leftmargin=*]
  \item \textbf{Position confidence: where to reconsider.}
        Following ReMDM-conf, we use the model's confidence in its clean proposal $\Ztilde{}$, read from the cavity or the denoiser, to concentrate revision on uncertain positions.
        We let $\kappa_i$ decrease with confidence at strength $\gamma$, keeping its mean equal to $\eta$ across positions.
        We also tilt the noise level $s_i$ to which re-corrupted positions return with the same confidence at strength $\lambda$, keeping the mean noise level at $s$.
  \item \textbf{Candidate confidence: which candidates to reconsider.}
        Plain re-corruption reopens alternative candidates uniformly at random, as the forward kernel does.
        For binary diffusion, we instead reopen candidate $u$ with probability proportional to $p_{i,u}^\beta$, keeping the expected number of reopened candidates fixed, where $p_{i,u}$ is the denoiser probability of $u$ when it was last in the candidate set.
        Uniform diffusion analogously draws the new token $u\ne\Ztilde{}$ in proportion to $\pi_i^\theta(u\mid Z_t;t)^\beta$, while masked diffusion has no candidate to reopen.
        Plausible candidates are thus preferentially restored.
\end{itemize}

\begin{algorithm}[hbt]
  \caption{Confidence-guided $\Psi$-sampler (full details specified in App.~\ref{app:psi-sampler}).}
  \label{alg:binary-sampling}
  \begin{algorithmic}[1]
    \Require Cavity model, kernel, grid $t_2=\tau_N>\cdots>\tau_0=t_1$, re-corrupt fraction $\eta$, strengths $\gamma,\lambda,\beta$, temperature $T$
    \State Draw $Z_{t_2}^i\sim p_{t_2}$ and set $p_{i,u}=1/V$ for all $i,u$.
    \For{$j=N,\ldots,1$}
    \State Set $t=\tau_j$, $s=\tau_{j-1}$, evaluate $\mu_i^\theta(\cdot\mid Z_t;t)$, compute $\pi_i^\theta$ by~\eqref{eq:binary-clean-distribution}, and record $p_{i,u}={\pi_i^\theta(u\mid Z_t;t)}$ for $u\in Z_t^i$ (every $u$ for uniform diffusion).
    \State Draw a proposal $\Ztilde{}\propto\pi_i^\theta(\cdot\mid Z_t;t)^{1/T}$ and read its confidence $\mathrm{conf}_i$ from $\mu_i^\theta$ or $\pi_i^\theta$.
    \State Set $\kappa_i\propto\exp(-\gamma\,\mathrm{conf}_i)$ and $s_i\propto\exp(-\lambda\,\mathrm{conf}_i)$, capped at one, with means $\eta$ and $s$ over positions (using equations~\eqref{eq:app-position-confidence} and~\eqref{eq:app-ordering-tilt}, App.~\ref{app:psi-sampler}).
    \State With probability $\kappa_i$, re-corrupt $\Ztilde{}$ to $s_i$, reopening candidates $u$ with probability~$\propto p_{i,u}^\beta$, and otherwise draw $Z_s^i$ from the ancestral update~\eqref{eq:binary-ancestral} with its own proposal $\Zhat{}$.
    \EndFor
    \State Sample each position $\propto\pi_i^\theta(\cdot\mid Z_{t_1};t_1)^{1/T}$, only the masked ones for masked diffusion.
  \end{algorithmic}
\end{algorithm}

\section{Experiments}

We focus on solve rate and sample quality and report negative ELBOs only for completeness: the negative ELBO measures the path-space KL to the exact reverse process~\citep{casado2026discretediffusionlearn}, but it need not track the quality of the final samples, especially under non-ancestral samplers, as uniform versus masked diffusion shows~\citep{schiff2025udlm,sahoo2025diffusionduality}.

\subsection{Logic and planning tasks}

We evaluate on Sudoku-Extreme~\citep{wang2025hrm},
Boxoban~\citep{guez2018boxoban,garrigaalonso2024sokoban},
and four-coloring of random graphs, using matched 43M-parameter diffusion transformers and five training seeds per kernel and task.
Besides the full test sets, we report hard subsets to de-saturate the benchmarks: low-clue Sudoku, long-solution Boxoban, and graphs ten times the training size.
A sample counts as solved only if it satisfies every constraint, with one attempt per instance and no verifier.
For the full and hard settings separately, we select each kernel's sampler on validation at a reference budget of 256 network evaluations (NFE) for Sudoku and Boxoban and 64 for graph coloring, with the same tuning budget for every kernel and candidate confidence for all but masked diffusion.
Parentheses below are standard errors over the five training seeds.
Full experimental details are given in App.~\ref{app:experimental-details}, and the test negative ELBOs in Table~\ref{tab:logic-test-nelbo}.

\begin{table}[hb]
  \centering
  \caption{Solve rate (\%) on the full test sets and hard subsets at the reference NFE.}
  \label{tab:logic-final}
{\footnotesize
\setlength{\tabcolsep}{4pt}
\begin{tabular}{@{}lrrrrrr@{}}
\toprule
 & \multicolumn{2}{c}{Sudoku} & \multicolumn{2}{c}{Boxoban} & \multicolumn{2}{c}{Graph coloring} \\
\cmidrule(lr){2-3}\cmidrule(lr){4-5}\cmidrule(lr){6-7}
Kernel & All & Hard & All & Hard & All & Hard \\
\midrule
Binary symmetric & \textbf{96.27} (0.21) & \textbf{92.66} (0.39) & \textbf{94.29} (0.30) & \textbf{84.65} (0.56) & \textbf{99.35} (0.00) & \textbf{64.13} (1.10) \\
Uniform & 95.47 (0.21) & 89.81 (0.47) & 92.38 (0.64) & 78.73 (1.20) & 99.01 (0.02) & 61.41 (0.34) \\
Binary absorbing & 91.24 (0.11) & 81.08 (0.26) & 87.30 (0.81) & 69.13 (1.66) & 98.83 (0.02) & 54.77 (0.39) \\
Masked & 87.96 (0.16) & 76.12 (0.10) & 91.26 (0.31) & 75.47 (0.69) & 97.44 (0.04) & 4.85 (0.33) \\
\bottomrule
\end{tabular}
}

\end{table}

\begin{figure}[hbt]
  \centering
  \includegraphics[width=\linewidth]{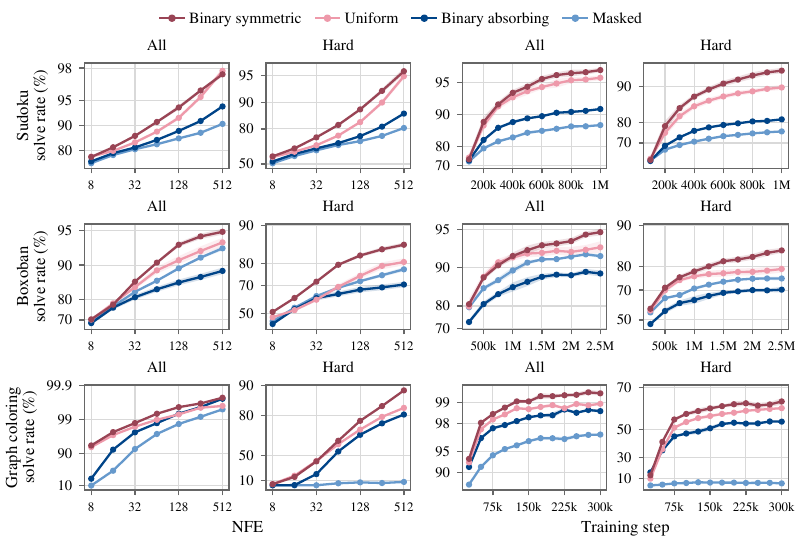}
  \caption{Solve rate against inference budget (NFE, left) and training step (right).}
  \label{fig:logic-nfe}
\end{figure}

Table~\ref{tab:logic-final} reports the solve rates at the reference NFE, while
Figure~\ref{fig:logic-nfe} shows their dependence on inference budget and training progress.
Binary symmetric diffusion is the strongest kernel on every task, followed by uniform diffusion, with masked and binary absorbing diffusion behind both.
Its margin widens on the hard subsets, especially on long-horizon Boxoban and on large graphs.
On the large graphs, masked diffusion nearly fails.
Binary symmetric diffusion leads at almost all budgets and training steps (Figure~\ref{fig:logic-nfe}), rather than depending on a single NFE or checkpoint.
Table~\ref{tab:logic-sampling} ablates the sampling components, while
Figure~\ref{fig:logic-difficulty} in App.~\ref{app:logic-tasks} shows that
the kernel comparison persists across difficulty bins and graph sizes.
We show how each kernel solves a Sudoku board in App.~\ref{app:sudoku-trajectories}.

\subsection{Text}

We compare the four kernels using unconditional generation on three corpora of increasing sequence length $L$: Text8~\citep{mahoney2011text8} at $L=256$, TinyStories~\citep{eldan2023tinystories} at $L=512$, and OpenWebText~\citep{gokaslan2019openwebtext} at $L=1024$ (App.~\ref{app:text-datasets}).
We train every model on 131B tokens, on Text8 at the character level with $V=27$ and on TinyStories and OpenWebText with byte-pair-encoding tokenizers of $V=1024$, using as training objective the path term of the negative ELBO~\eqref{eq:binary-elbo} (except for masked and binary absorbing diffusion on Text8, see App.~\ref{app:training-objectives}).
We report MAUVE~\citep{pillutla2021mauve} against GPT-2 token unigram entropy, with GPT-2 Large~\citep{radford2019gpt2} supplying the features.
MAUVE compares the distribution of generated samples with held-out text, and we use it as the primary metric given known failure modes of generative perplexity (GenPPL)~\citep{franca2026hackingperplexity,zheng2025timeagnostic}, which we report with the negative ELBOs in App.~\ref{app:text-datasets}.

Table~\ref{tab:text-final} gives the test MAUVE of each kernel at the sampler setting selected on validation (App.~\ref{app:sampler-tuning}), while Figure~\ref{fig:text-frontier} traces quality against diversity using sampling settings selected on validation at representative entropies and evaluated on the test split.
No kernel leads on all three corpora: binary symmetric diffusion reaches the best MAUVE on Text8, uniform diffusion on TinyStories, and masked diffusion on OpenWebText, each by a small margin.
The kernels peak at slightly different entropies, so Table~\ref{tab:text-common} in App.~\ref{app:text-datasets} also compares them at one common entropy per corpus, where binary symmetric diffusion reaches the best MAUVE on both Text8 and TinyStories.

\begin{table}[h]
  \centering
  \caption{Test MAUVE of tuned samplers (standard error over two generation seeds).}
  \label{tab:text-final}
{\footnotesize
\setlength{\tabcolsep}{6pt}
\begin{tabular}{@{}lrrrr@{}}
\toprule
Dataset & Binary symmetric & Uniform & Binary absorbing & Masked \\
\midrule
Text8 & \textbf{0.913} (0.001) & 0.896 (0.006) & 0.821 (0.000) & 0.854 (0.008) \\
TinyStories & 0.882 (0.006) & \textbf{0.904} (0.001) & 0.585 (0.011) & 0.850 (0.009) \\
OpenWebText & 0.840 (0.001) & 0.822 (0.004) & 0.747 (0.004) & \textbf{0.871} (0.013) \\
\bottomrule
\end{tabular}
}

\end{table}

\begin{figure}[h]
  \centering
  \includegraphics[width=\linewidth]{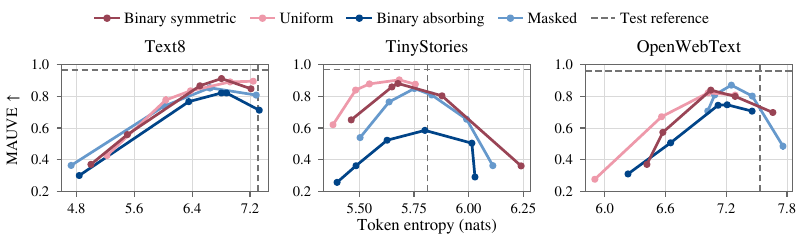}
  \caption{MAUVE against token entropy with GPT-2 Large as the scorer. We select each point's sampling configuration on validation at a representative entropy and evaluate it on the test split.}
  \label{fig:text-frontier}
\end{figure}

\section{Related work}

Discrete diffusion was first formulated on a discrete time grid, modeling the reverse transitions directly~\citep{hoogeboom2021argmax,austin2021d3pm}, and its continuous-time formulation instead learns to reverse a categorical jump process~\citep{campbell2022continuous,sun2022score}.
Since then, methods have differed mainly in how they parameterize the reverse rates~\citep{lou2024sedd,ou2025radd} and in which paths the forward process allows~\citep{vonruette2025gidd,shaul2025kinetic}.
Closer to our setting, simplex-valued methods represent categorical uncertainty in a continuous state~\citep{floto2023simplex,avdeyev2023dirichlet}, while partial and semantic masking keep residual information about the clean token~\citep{chao2025partialmasking,zhou2025hdlm}.
Binary diffusion instead changes the state space: each position follows a jump process on the binary cube, whose coordinates express candidate membership.
Similar bit-flipping processes have been explored to model binary data, where the cube is the data space~\citep{sohldickstein2015deep,pham2025dmpm}.
The closest predecessor is the ShortListing Model, which also grows candidate sets, in discrete time~\citep{song2025shortlisting}.
Under the right noise schedule, its corruption process provides a discrete-time formulation of binary absorbing diffusion (App.~\ref{app:extended-related-work}).
Our main contributions are the non-absorbing symmetric process, which allows candidates to be both removed and restored during denoising, and the empirical evaluation on logic and planning tasks.

Discrete diffusion has been applied to Sudoku and other structured problems before~\citep{ye2024beyondautoregression,wewer2025spatialreasoning}, and follow-up work improves solving through adaptive token order, remasking, explicit planning, or test-time search~\citep{kim2025tokenordering,wang2025remaskingdiscretediffusion,liu2025planneddenoising,zhang2025vfscale}.
Sudoku is also a standard benchmark for small recurrent reasoners such as HRM and TRM~\citep{wang2025hrm,trm2025}, and Lattice Deduction Transformers likewise maintain per-variable candidate sets but train deduction on-policy with search and abstention heuristics~\citep{davis2026latticedeductiontransformers}.
All of these change the sampler, the architecture, or the training signal.
Our experiments are instead designed to isolate the effect of the corruption kernel and evaluate direct, verifier-free samples from each kernel's tuned sampler.

\section{Conclusion}

We have introduced Binary Discrete Diffusion, a family of continuous-time discrete-diffusion processes whose noisy states are sets of candidate tokens, represented as binary membership vectors.
Parameterizing the reverse process through clean-token predictions lets a single $V$-way output determine all reverse rates and sampling updates, avoiding any $2^V$ enumeration in training and generation.
We also introduced a confidence-guided self-correcting sampler that chooses which positions and candidates to reopen.
Under controlled architectures, training budgets, and sampler tuning, binary symmetric diffusion outperforms uniform diffusion, and both outperform masked and binary absorbing diffusion on logic and planning tasks, with the largest gains on the hardest instances.
Binary diffusion also remains competitive on language modeling.
Our results thus show that candidate sets can be an effective intermediate representation for iterative reasoning, and that the corruption process can materially change the computation learned by the model.

Binary diffusion also opens a broader design space for candidate-set and bitwise corruption processes.
Future kernels could exploit semantic or task-specific structure to introduce plausible candidates rather than random ones.
The bits could also encode tokens differently, for example through their binary indices or an error-correcting code, although each such code defines a different state space whose states are no longer candidate sets.
Richer representations of candidate sets than the normalized embedding sum used here may also prove useful at larger scales.
Scaling these ideas to larger models and datasets, and to domains such as code, mathematics, and biological sequence design, is a natural next step.

\subsection*{AI use statement}

Large language models were used to aid and polish writing (editing and reviewing the manuscript), for retrieval and discovery (finding and summarizing related work), and for research ideation and execution (implementing, running, and analyzing experiments).
All AI-assisted work was reviewed by the authors, who verified every technical claim and result and take full responsibility for the final content.



\subsection*{Reproducibility statement}

App.~\ref{app:experimental-details} specifies the datasets with their splits and hard subsets, the architectures and training settings, every evaluation procedure, and the selection of learning rates, checkpoints, and sampler settings.
App.~\ref{app:training} describes the training objectives and time sampling, Apps.~\ref{app:masked-uniform} and~\ref{app:binary} derive the rates, losses, and ancestral updates of every kernel, and App.~\ref{app:psi-sampler} specifies the sampler.

\bibliography{ref}
\bibliographystyle{iclr2027_conference}

\clearpage
\appendix
\section*{Appendix contents}
\startcontents[appendix]
\printcontents[appendix]{l}{1}{\setcounter{tocdepth}{2}}

\newpage

\section{Masked and uniform diffusion}
\label{app:masked-uniform}

We specialize the results of~\citet{campbell2022continuous} and~\citet{casado2026discretediffusionlearn} using the notation of Sections~\ref{sec:discrete}--\ref{sec:binary-sampling}.
Local noisy states are written $z,y$, and the priors are the limiting distributions from which generation starts.
We call the noisy-state distribution $m_i^\theta$ in~\eqref{eq:binary-rate-parameterization} the local model mass.

\subsection{Masked diffusion}
\label{app:masked}

For the masked kernel in~\eqref{eq:standard-kernels}, the off-diagonal forward rates are~\citep{casado2026discretediffusionlearn}
\begin{equation}
  Q_t(z,y)=\frac{\mathbf1\{z\in[V],\ y=\mathtt m\}}{1-t},\qquad y\ne z.
  \label{eq:app-masked-rates-forward}
\end{equation}
For $0\le s<t<1$, the transition kernel is
\begin{equation}
  q_{t\mid s}(y\mid z)
  =\frac{1-t}{1-s}\,\mathbf1\{y=z\}
  +\frac{t-s}{1-s}\,\mathbf1\{y=\mathtt m\}.
  \label{eq:app-masked-kernel}
\end{equation}
The local model masses and denoiser are
\begin{equation}
  \begin{aligned}
    m_i^\theta(y\mid Z_t;t)
     & =
    \begin{cases}
      (1-t)\,\mu_i^\theta(y\mid Z_t;t), & y\in[V],     \\
      t,                                & y=\mathtt m,
    \end{cases}
    \\
    \pi_i^\theta(c\mid Z_t;t)
     & =
    \begin{cases}
      \mu_i^\theta(c\mid Z_t;t), & Z_t^i=\mathtt m, \\
      \mathbf1\{Z_t^i=c\},       & Z_t^i\in[V].
    \end{cases}
  \end{aligned}
  \label{eq:app-masked-head}
\end{equation}
For $y\in[V]$ with $y\ne Z_t^i$, the reverse rates are
\begin{equation}
  a_i(Z_0,Z_t,y;t)
  =\frac{\mathbf1\{Z_t^i=\mathtt m,\ y=Z_0^i\}}t,
  \qquad
  a_i^\theta(Z_t,y;t)
  =\frac{\mathbf1\{Z_t^i=\mathtt m\}\,\mu_i^\theta(y\mid Z_t;t)}t.
  \label{eq:app-masked-rates}
\end{equation}
All other off-diagonal reverse rates vanish.
The path loss reduces to masked cross-entropy:
\begin{equation}
  \mathcal J_t(\theta;Z_0,Z_t)
  =-\frac1t\sum_i\mathbf1\{Z_t^i=\mathtt m\}
  \log\mu_i^\theta(Z_0^i\mid Z_t;t).
  \label{eq:app-masked-loss}
\end{equation}
Plugging $\mu_i^\theta$ into the bridge~\eqref{eq:app-plug-in-bridge} gives the ancestral update of \citet{sahoo2024mdlm}: visible tokens are copied, and a masked position unmasks to $c$ with probability $(1-s/t)\,\mu_i^\theta(c\mid Z_t;t)$.

\subsection{Uniform diffusion}
\label{app:uniform}

For the uniform kernel in~\eqref{eq:standard-kernels}, the off-diagonal forward rates are~\citep{casado2026discretediffusionlearn}
\begin{equation}
  Q_t(z,y)=\frac1{V(1-t)},\qquad y\ne z.
  \label{eq:app-uniform-rates-forward}
\end{equation}
For $0\le s<t<1$, the transition kernel is
\begin{equation}
  q_{t\mid s}(y\mid z)
  =\frac{1-t}{1-s}\,\mathbf1\{y=z\}
  +\frac{t-s}{V(1-s)}.
  \label{eq:app-uniform-kernel}
\end{equation}
The local model mass and induced denoiser are
\begin{equation}
  \begin{aligned}
    m_i^\theta(y\mid Z_t;t) & =(1-t)\,\mu_i^\theta(y\mid Z_t;t)+t/V,                             \\
    \pi_i^\theta(c\mid Z_t;t)
                            & =\frac{\mu_i^\theta(c\mid Z_t;t)\,[(1-t)\,\mathbf1\{Z_t^i=c\}+t/V]}
                               {(1-t)\,\mu_i^\theta(Z_t^i\mid Z_t;t)+t/V}.
  \end{aligned}
  \label{eq:app-uniform-head}
\end{equation}
For $y\ne Z_t^i$, the clean-conditioned and model reverse rates are
\begin{equation}
  \begin{aligned}
    a_i(Z_0,Z_t,y;t)
     & =\frac1{V(1-t)}
    \frac{(1-t)\,\mathbf1\{y=Z_0^i\}+t/V}
    {(1-t)\,\mathbf1\{Z_t^i=Z_0^i\}+t/V}, \\
    a_i^\theta(Z_t,y;t)
     & =\frac1{V(1-t)}
    \frac{(1-t)\,\mu_i^\theta(y\mid Z_t;t)+t/V}
    {(1-t)\,\mu_i^\theta(Z_t^i\mid Z_t;t)+t/V}.
  \end{aligned}
  \label{eq:app-uniform-rates}
\end{equation}
The path loss is
\begin{equation}
  \mathcal J_t(\theta;Z_0,Z_t)
  =\frac1{V(1-t)}\sum_i\sum_{y\ne Z_t^i}
  \Phi\left(
  \frac{q_{t\mid0}(y\mid Z_0^i)}{q_{t\mid0}(Z_t^i\mid Z_0^i)},
  \frac{m_i^\theta(y\mid Z_t;t)}{m_i^\theta(Z_t^i\mid Z_t;t)}
  \right).
  \label{eq:app-uniform-loss}
\end{equation}
All masses in these quotients are positive on an interior time interval.

Plugging $\mu_i^\theta$ into the bridge~\eqref{eq:app-plug-in-bridge} gives the ancestral update of \citet{schiff2025udlm}, which, unlike the masked update, can change a visible token.
\paragraph{Prior term.}
Corrupting a uniform clean token leaves its distribution uniform.
The finite-time and limiting priors thus coincide:
\begin{equation}
  p_{t_2}(y)=1/V,\qquad y\in[V].
  \label{eq:app-uniform-prior}
\end{equation}
For the product sequence prior,
\begin{equation}
  \mathbb E_{Z_0}\operatorname{KL}\big(q_{t_2\mid0}(\cdot\mid Z_0)\,\|\,p_{t_2}\big)
  =L\left[
    (1-t_2+t_2/V)\log\big(V(1-t_2)+t_2\big)
    +\frac{(V-1)t_2}{V}\log t_2
    \right].
  \label{eq:app-uniform-prior-kl}
\end{equation}
The uniform prior has full support at every finite cutoff.

\section{Mathematical details of binary diffusion}
\label{app:binary}

We use the candidate-set notation and the bitwise loss from Section~\ref{sec:binary}.
The network probability mass of the candidates in a set $z$ is
\begin{equation}
  \nu_i^\theta(z\mid Z_t;t)
  \defeq\sum_{c\in z}\mu_i^\theta(c\mid Z_t;t),
  \label{eq:binary-symmetric-set-mass}
\end{equation}
and $m_i^\theta$ denotes the local model mass of~\eqref{eq:binary-rate-parameterization}.
Superscripts $\mathrm{sym}$ and $\mathrm{abs}$ specify the kernel, and we omit them in derivations where it is fixed.
The cavity $\mu_i^\theta$, the masses $m_i^\theta$, and the denoiser $\pi_i^\theta$ always refer to that kernel.

\subsection{Binary symmetric diffusion}
\label{app:binary-symmetric}

\begin{proposition}[Symmetric kernel and reverse rates]
  \label{prop:app-binary-symmetric}
  Let each bit flip at rate $1/[2(1-t)]$.
  For $0\le s<t<1$, the probability of a bit differing from its value at $s$ is
  \begin{equation}
    \varepsilon_{s,t}=\frac{t-s}{2(1-s)}.
    \label{eq:app-binary-symmetric-flip}
  \end{equation}
  The local kernel is
  \begin{equation}
    q_{t\mid s}^{\mathrm{sym}}(y\mid z)
    =\big(1-\varepsilon_{s,t}\big)^{V-d_H(y,z)}
    \big(\varepsilon_{s,t}\big)^{d_H(y,z)}.
    \label{eq:app-binary-symmetric-kernel}
  \end{equation}
  Using the odds $\rho_t$ from~\eqref{eq:binary-symmetric-odds} and writing $z=Z_t^i$, the reverse rates are
  \begin{equation}
    \begin{aligned}
      a_{i,u}^{\mathrm{sym}}(Z_0,Z_t;t)
       & =\frac1{2(1-t)}
      \begin{cases}
        \rho_t,      & z_u\ne(Z_0^i)_u, \\
        \rho_t^{-1}, & z_u=(Z_0^i)_u,
      \end{cases}
      \\
      a_{i,u}^{\theta,\mathrm{sym}}(Z_t;t)
       & =\frac{\rho_t^{2z_u-1}}{2(1-t)}
      \frac{1+(\rho_t^2-1)\,\nu_i^\theta(z^{(u)}\mid Z_t;t)}
      {1+(\rho_t^2-1)\,\nu_i^\theta(z\mid Z_t;t)}.
    \end{aligned}
    \label{eq:app-binary-symmetric-rates}
  \end{equation}
\end{proposition}
\begin{proof}
  The bit-flip probability solves
  \[
    \partial_t \varepsilon_{s,t}=\frac{1-2\varepsilon_{s,t}}{2(1-t)},\qquad \varepsilon_{s,s}=0.
  \]
  The stated flip probability satisfies this equation and initial condition, and multiplying the independent bit kernels proves~\eqref{eq:app-binary-symmetric-kernel}.
  The clean-conditioned reverse rates follow from the change in Hamming distance: a flip toward the clean bit multiplies the clean-conditioned mass by $\rho_t$, while a flip away multiplies it by $\rho_t^{-1}$.
  Multiplying these ratios by the forward flip rate gives the formula for $a_{i,u}$.

  To compute the model mass, put
  \[
    C_t(z)=(1-t/2)^{V-|z|}(t/2)^{|z|}.
  \]
  As a function of the clean token, the forward kernel is
  \[
    q_{t\mid0}(z\mid c)
    =C_t(z)\,\rho_t^{2z_c-1},
    \qquad
    m_i^\theta(z\mid Z_t;t)=\frac{C_t(z)}{\rho_t}
    [1+(\rho_t^2-1)\,\nu_i^\theta(z\mid Z_t;t)].
  \]
  Toggling bit $u$ gives
  \[
    \frac{C_t(z^{(u)})}{C_t(z)}=\rho_t^{2z_u-1},\qquad
    \nu_i^\theta(z^{(u)}\mid Z_t;t)=\nu_i^\theta(z\mid Z_t;t)+(1-2z_u)\,\mu_i^\theta(u\mid Z_t;t).
  \]
  The quotient of the two local masses, multiplied by $1/[2(1-t)]$, gives the stated model rate.
  All local masses are positive for $0<t<1$, so every bit flip has positive model rate.
\end{proof}

The denoiser used for decoding and sampling is
\begin{equation}
  \pi_i^\theta(c\mid Z_t;t)
  =\frac{\mu_i^\theta(c\mid Z_t;t)\,\rho_t^{2(Z_t^i)_c}}
  {1+(\rho_t^2-1)\,\nu_i^\theta(Z_t^i\mid Z_t;t)}.
  \label{eq:app-binary-symmetric-decode}
\end{equation}
Substituting $\mu_i^\theta(u\mid Z_t;t)=\rho_t^{-2z_u}\pi_u[1+(\rho_t^2-1)\nu_i^\theta(z\mid Z_t;t)]$, with $\pi_u=\pi_i^\theta(u\mid Z_t;t)$, turns the ratio of masses in~\eqref{eq:app-binary-symmetric-rates} into $1-(1-\rho_t^{-2})\pi_u$ for $u\in z$ and $1+(\rho_t^2-1)\pi_u$ for $u\notin z$.
The model rate is therefore $[\rho_t(1-\pi_u)+\rho_t^{-1}\pi_u]/[2(1-t)]$ for removing an active candidate and $[\rho_t\pi_u+\rho_t^{-1}(1-\pi_u)]/[2(1-t)]$ for adding an inactive one, which is~\eqref{eq:binary-symmetric-removal}.

\paragraph{Prior.}
The prior draws independent fair bits, $p_{t_2}^{\mathrm{sym}}(z)=2^{-V}$.
For Bernoulli distributions with success probabilities $a,b\in(0,1)$, the KL divergence is
\begin{equation}
  D_{\mathrm{Ber}}(a\|b)
  =a\log\frac ab+(1-a)\log\frac{1-a}{1-b}.
  \label{eq:app-bernoulli-kl}
\end{equation}
Conditional on the clean token $c$, the bits of $Z_{t_2}^i$ are independent, with bit $c$ active with probability $1-t_2/2$ and every other bit with probability $t_2/2$.
The prior also factors over bits, so the KL is a sum of Bernoulli terms, which coincide because $D_{\mathrm{Ber}}(1-a\|1/2)=D_{\mathrm{Ber}}(a\|1/2)$:
\begin{equation}
  \operatorname{KL}\big(q_{t_2\mid0}^{\mathrm{sym}}(\cdot\mid c)\|p_{t_2}^{\mathrm{sym}}\big)
  =V D_{\mathrm{Ber}}(t_2/2\|1/2).
  \label{eq:app-binary-symmetric-prior-kl}
\end{equation}
This does not depend on $c$ and vanishes as $t_2\uparrow1$.

\subsection{Binary absorbing diffusion}
\label{app:binary-absorbing}

\begin{proposition}[Absorbing kernel and rates]
  \label{prop:app-binary-absorbing}
  Let each zero bit activate at rate $1/(1-t)$ and each active bit remain active.
  For $0\le s<t<1$ and nonempty binary vectors $z,y$,
  \begin{equation}
    q_{t\mid s}^{\mathrm{abs}}(y\mid z)
    =\mathbf1\{z\subseteq y\}
    \left(\frac{1-t}{1-s}\right)^{V-|y|}
    \left(\frac{t-s}{1-s}\right)^{|y|-|z|}.
    \label{eq:app-binary-absorbing-kernel}
  \end{equation}
  The expression is zero when the subset condition fails.
  For the clean singleton $\{c\}$, it reduces to
  \begin{equation}
    q_{t\mid0}^{\mathrm{abs}}(z\mid c)
    =\mathbf1\{c\in z\}(1-t)^{V-|z|}t^{|z|-1}.
    \label{eq:app-binary-absorbing-clean-kernel}
  \end{equation}
  The clean-conditioned and model reverse rates are
  \begin{equation}
    a_{i,u}^{\mathrm{abs}}(Z_0,Z_t;t)=\frac{\mathbf1\{u\in Z_t^i\setminus Z_0^i\}}t,
    \qquad
    a_{i,u}^{\theta,\mathrm{abs}}(Z_t;t)
    =\frac{\mathbf1\{u\in Z_t^i\}}t
    \frac{\nu_i^\theta(Z_t^i\setminus\{u\}\mid Z_t;t)}{\nu_i^\theta(Z_t^i\mid Z_t;t)}.
    \label{eq:app-binary-absorbing-rates}
  \end{equation}
\end{proposition}
\begin{proof}
  A bit that is zero at $s$ remains zero at $t$ with probability
  \[
    \exp\left(-\int_s^t\frac{dv}{1-v}\right)=\frac{1-t}{1-s}.
  \]
  Once activated, it remains one.
  Multiplication over bits proves~\eqref{eq:app-binary-absorbing-kernel}, and $z=\{c\}$, $s=0$ gives~\eqref{eq:app-binary-absorbing-clean-kernel}.
  For a nonempty state $z$,
  \[
    m_i^\theta(z\mid Z_t;t)
    =(1-t)^{V-|z|}t^{|z|-1}\,\nu_i^\theta(z\mid Z_t;t).
  \]
  Deleting an active bit $u$ changes this mass by the factor
  \[
    \frac{(1-t)\,\nu_i^\theta(z\setminus\{u\}\mid Z_t;t)}
    {t\,\nu_i^\theta(z\mid Z_t;t)}.
  \]
  Multiplying by the reversed forward activation rate $1/(1-t)$ gives the model rate.
  For $Z_0^i=c$, the substitution
  \[
    \mu_i^\theta(u\mid Z_t;t)=\mathbf1\{u=c\}
  \]
  gives the clean-conditioned rate.
  Deleting the last bit has zero rate because the empty vector has zero mass.
\end{proof}

\begin{proposition}[Absorbing path loss]
  \label{prop:app-binary-absorbing-loss}
  Under the assumptions of Proposition~\ref{prop:app-binary-absorbing},
  \begin{equation}
    \mathcal J_t^{\mathrm{abs}}(\theta;Z_0,Z_t)
    =-\frac1t\sum_i\sum_{u\in Z_t^i\setminus Z_0^i}
    \log\left(1-\frac{\mu_i^\theta(u\mid Z_t;t)}{\nu_i^\theta(Z_t^i\mid Z_t;t)}\right).
    \label{eq:app-binary-absorbing-loss}
  \end{equation}
\end{proposition}
\begin{proof}
  Fix $i$, write $Z_0^i=c$, and put
  \[
    h_u=\frac{\mu_i^\theta(u\mid Z_t;t)}{\nu_i^\theta(Z_t^i\mid Z_t;t)},\quad u\in Z_t^i,
    \qquad \sum_{u\in Z_t^i}h_u=1.
  \]
  The homogeneity identity $\Phi(a/t,b/t)=\Phi(a,b)/t$ gives the local loss times $t$ as
  \[
    \Phi(0,1-h_c)
    +\sum_{u\in Z_t^i\setminus\{c\}}\Phi(1,1-h_u)
    =1-h_c+\sum_{u\ne c,\ u\in Z_t^i}
    [-h_u-\log(1-h_u)].
  \]
  Because the $h_u$ sum to one, the linear terms cancel, leaving~\eqref{eq:app-binary-absorbing-loss}.
  For a singleton, the remaining sum is empty and the loss is zero.
\end{proof}

For absorbing diffusion, the local kernel is constant over clean tokens in $Z_t^i$ and zero outside it.
Bayes reweighting gives the denoiser
\begin{equation}
  \pi_i^\theta(c\mid Z_t;t)
  =\frac{\mu_i^\theta(c\mid Z_t;t)\,\mathbf1\{c\in Z_t^i\}}
  {\nu_i^\theta(Z_t^i\mid Z_t;t)}.
  \label{eq:app-binary-absorbing-decode}
\end{equation}
Since $\nu_i^\theta(Z_t^i\setminus\{u\}\mid Z_t;t)/\nu_i^\theta(Z_t^i\mid Z_t;t)=1-\pi_i^\theta(u\mid Z_t;t)$ for $u\in Z_t^i$, the model rate in~\eqref{eq:app-binary-absorbing-rates} equals~\eqref{eq:binary-absorbing-rate}.
The cavity parameterization applies this support restriction analytically, while a separately predicted denoiser must respect it through explicit masking or learning.

\subsection{Binary ancestral bridges}
\label{app:binary-sampling}

By Bayes' rule and the Markov property, the forward bridge of Section~\ref{sec:binary-sampling} is
\begin{equation}
  q_{s\mid t,0}(y\mid z,c)
  =\frac{q_{t\mid s}(z\mid y)\,q_{s\mid0}(y\mid c)}{q_{t\mid0}(z\mid c)}.
  \label{eq:discrete-bridge}
\end{equation}
Averaging it over the denoiser~\eqref{eq:binary-clean-distribution} cancels $q_{t\mid0}(z\mid c)$, which gives the plug-in form of the ancestral update~\eqref{eq:binary-ancestral}:
\begin{equation}
  p_{s\mid t,i}^\theta(y\mid Z_t)
  =\frac{q_{t\mid s}(Z_t^i\mid y)\sum_c\mu_i^\theta(c\mid Z_t;t)\,q_{s\mid0}(y\mid c)}
  {\sum_c\mu_i^\theta(c\mid Z_t;t)\,q_{t\mid0}(Z_t^i\mid c)}.
  \label{eq:app-plug-in-bridge}
\end{equation}
This is the bridge with the cavity in place of the clean token.
We now derive the conditional bit bridges from~\eqref{eq:discrete-bridge}.

\paragraph{Symmetric bridge.}
Conditional on the clean proposal $\Zhat{}=c$, the bits are independent.
With $\varepsilon_{s,t}$ from~\eqref{eq:app-binary-symmetric-flip}, bit $u$ with clean value $b_u=\mathbf1\{u=c\}$ and current value $z_u=(Z_t^i)_u$ differs from $b_u$ at $s$ with probability
\begin{equation}
  q^{\mathrm{sym}}\big((Z_s^i)_u\ne b_u\mid z_u,b_u\big)
  =
  \begin{cases}
    s\,\varepsilon_{s,t}/(2-t), & z_u=b_u,    \\[1mm]
    s\,(1-\varepsilon_{s,t})/t, & z_u\ne b_u,
  \end{cases}
  \label{eq:app-binary-symmetric-bridge}
\end{equation}
the chance of differing from the clean bit at $s$ and then flipping, or not, to the current value at $t$.
It favors the current value for $s$ close to $t$ and the clean bit as $s\downarrow0$.

A simpler symmetric sampler applies a one-bit Bayes update to each cavity bit marginal and draws the bits independently.
This ignores the dependence between the bits through the shared clean token: at $s=0$, the exact bridge draws a single one-hot state from $\pi_i^\theta$, whereas independent bit draws can leave zero or several candidates active.
It performed worse in early experiments, and we dropped it.

\paragraph{Absorbing bridge.}
Conditional on $\Zhat{}=c$, the bit $c$ stays active, and every bit inactive in $Z_t^i$ remains inactive at $s$.
For $u\in Z_t^i\setminus\{c\}$,
\begin{equation}
  q^{\mathrm{abs}}((Z_s^i)_u=1\mid (Z_t^i)_u=1,Z_0^i=c)=\frac{s}{t}.
  \label{eq:app-binary-absorbing-bridge}
\end{equation}
Indeed, a false candidate is active at $s$ with probability $s$, active at $t$ with probability $t$, and activation by $s$ implies activation by $t$.
These bit events are independent conditional on $c$ and the current state.
The update therefore preserves nonemptiness and produces a subset of the current active set.

\section{Training objectives and time sampling}
\label{app:training}

\subsection{Finite-window negative ELBO and boundary terms}
\label{app:finite-window-nelbo}

The idealized objective~\eqref{eq:binary-elbo} integrates the reverse process from clean data at $t=0$ to the terminal noise distribution at $t=1$.
In practice we use an interior window $0<t_1<t_2<1$ to avoid singular jump rates near the endpoints.
The corresponding finite-window negative ELBO, abbreviated NELBO in the appendices, is
\begin{equation}
  \NELBO(\theta;t_1,t_2)
  =
  \mathcal R_{t_1}(\theta)
  +
  \mathcal K_{t_2}
  +
  \int_{t_1}^{t_2}
  \mathbb E_{Z_0,Z_t}
    [\mathcal J_t(\theta;Z_0,Z_t)]\,dt,
  \label{eq:app-finite-window-nelbo}
\end{equation}
where
\begin{equation}
  \mathcal R_{t_1}(\theta)
  =
  \mathbb E_{Z_0,Z_{t_1}}
    [-\log p_{0\mid t_1}^\theta(Z_0\mid Z_{t_1})],
  \qquad
  \mathcal K_{t_2}
  =
  \mathbb E_{Z_0}
  \operatorname{KL}
  \big(
  q_{t_2\mid0}(\cdot\mid Z_0)
  \,\|\,p_{t_2}
  \big).
  \label{eq:app-boundary-terms}
\end{equation}
The first term, with the denoiser as decoder, $p_{0\mid t_1}^\theta(Z_0\mid Z_{t_1})=\prod_i\pi_i^\theta(Z_0^i\mid Z_{t_1};t_1)$, accounts for the residual corruption at the clean end of the truncated path, while the second measures the mismatch between the forward distribution at $t_2$ and the prior from which generation starts.

For uniform and binary symmetric diffusion, the terminal-prior term is finite and available analytically (Apps.~\ref{app:uniform} and~\ref{app:binary-symmetric}).
With the prior matched to the limiting forward distribution, the boundary contributions vanish as $t_1\downarrow0$ and $t_2\uparrow1$, recovering~\eqref{eq:binary-elbo}.
For masked and binary absorbing diffusion, the limiting prior is a point mass, so the terminal KL is infinite for every $t_2<1$ but zero at $t_2=1$.
The terminal endpoint therefore requires separate treatment, and App.~\ref{app:likelihood-evaluation} describes our likelihood-evaluation convention.

\subsection{Training objectives}
\label{app:training-objectives}
Alongside the path ELBO, we compared a flat ELBO, cross-entropy (xent), and weighted cross-entropy (wxent).
For the kernels considered here, the reverse rates share a scalar prefactor, which extends the flip rate $r_t$ of~\eqref{eq:binary-symmetric-forward-rates} to every kernel,
\begin{equation}
  r_t=
  \begin{cases}
    1/[2(1-t)], & \text{binary symmetric}, \\
    1/[V(1-t)], & \text{uniform},          \\
    1/t,        & \text{binary absorbing}, \\
    1/t,        & \text{masked},
  \end{cases}
  \label{eq:training-rate-factor}
\end{equation}
and the remaining factors are ratios of the forward kernel $q_{t\mid0}$ at the current noisy token.
Since $\Phi(ra,rb)=r\Phi(a,b)$, the prefactor is also the weight of the jump penalties in the path integrand $\mathcal J_t$ of~\eqref{eq:binary-elbo}.
The flat ELBO drops it, $\mathcal J_t^{\mathrm{flat}}=\mathcal J_t/r_t$, which changes the relative weight of the noise levels.
Cross-entropy trains the network output directly as a denoiser,
\begin{equation}
  \mathcal J_t^{\mathrm{xent}}(\theta;Z_0,Z_t)
  =-\sum_{i=1}^L\log\mu_i^\theta(Z_0^i\mid Z_t;t),
  \label{eq:training-xent}
\end{equation}
and weighted cross-entropy multiplies it by the prefactor, $\mathcal J_t^{\mathrm{wxent}}=r_t\,\mathcal J_t^{\mathrm{xent}}$.
For masked diffusion the sums run over masked positions.
The ELBO and wxent are integrated over time as in~\eqref{eq:binary-elbo}, while the flat ELBO and xent are averaged over the base schedule of App.~\ref{app:binary-clocks}.
All four are estimated with the inverse-density weight defined there, which the adaptive proposal of App.~\ref{app:adaptive-time-sampling} preserves.

We train the logic models on the path ELBO.
For text, we selected the objective of each kernel and corpus on the validation NELBO after 10\% of the training budget.
The path ELBO was best everywhere except on Text8, where we train masked diffusion on xent, which coincides with its flat ELBO, and binary absorbing diffusion on wxent.
The latter still bounds the likelihood, as we show next.

\paragraph{Absorbing cross-entropy bound.}
For binary absorbing diffusion, wxent upper-bounds the path integrand.
Fix a position with clean token $Z_0^i=c$.
The denoiser is supported on the active set $Z_t^i$, which contains $c$, so
\begin{equation}
  \prod_{u\in Z_t^i\setminus\{c\}}
  \bigl(1-\pi_i^\theta(u\mid Z_t;t)\bigr)
  \ge 1-\sum_{u\in Z_t^i\setminus\{c\}}\pi_i^\theta(u\mid Z_t;t)
  =\pi_i^\theta(c\mid Z_t;t).
  \label{eq:app-absorbing-xent-bound}
\end{equation}
Since $\pi_i^\theta(c\mid Z_t;t)=\mu_i^\theta(c\mid Z_t;t)/\nu_i^\theta(Z_t^i\mid Z_t;t)\ge\mu_i^\theta(c\mid Z_t;t)$, taking negative logarithms, multiplying by $1/t$, and summing over positions in~\eqref{eq:app-binary-absorbing-loss} gives $\mathcal J_t^{\mathrm{abs}}\le\mathcal J_t^{\mathrm{wxent}}$.

\subsection{Time schedules}
\label{app:binary-clocks}

The base schedule reparameterizes corruption time on the clipped window $[t_1,t_2]$ by $\xi\in[0,1]$, which we draw uniformly during training.
The linear, logarithmic, and cosine maps are
\begin{equation}
  t_{\mathrm{lin}}(\xi)=t_1+(t_2-t_1)\xi,
  \qquad
  t_{\log}(\xi)=t_1\left(\frac{t_2}{t_1}\right)^{\xi},
  \qquad
  t_{\cos}(\xi)=t_1+(t_2-t_1)\sin^2\!\Big(\frac{\pi\xi}{2}\Big),
  \label{eq:app-binary-clock-maps}
\end{equation}
which give the base sampling densities on $[t_1,t_2]$
\begin{equation}
  p_{\mathrm{lin}}(t)=\frac1{t_2-t_1},
  \qquad
  p_{\log}(t)=\frac1{t\log(t_2/t_1)},
  \qquad
  p_{\cos}(t)=\frac1{\pi\sqrt{(t-t_1)(t_2-t)}}.
  \label{eq:app-binary-clock-densities}
\end{equation}
Logarithmic sampling assigns equal probability to intervals with the same ratio of endpoints, giving more resolution near clean data.
At large vocabularies, this matches how the binary kernels lose information (App.~\ref{app:information}).
Cosine sampling instead gives more resolution near both ends of the window.
To estimate the path integral without bias, we weight each sequence's loss by the inverse density of its sampled time, namely $t_2-t_1$ for the linear schedule, $t\log(t_2/t_1)$ for the logarithmic schedule, and $\pi\sqrt{(t-t_1)(t_2-t)}$ for the cosine schedule.

We give $\xi$ to the network as its time input:
\begin{equation}
  \xi(t)=
  \begin{cases}
    \displaystyle\frac{t-t_1}{t_2-t_1},                        & \text{linear schedule},      \\[6pt]
    \displaystyle\frac{\log(t/t_1)}{\log(t_2/t_1)},            & \text{logarithmic schedule}, \\[6pt]
    \displaystyle\frac2\pi\arcsin\sqrt{\frac{t-t_1}{t_2-t_1}}, & \text{cosine schedule}.
  \end{cases}
  \label{eq:app-binary-clock-conditioning}
\end{equation}
The adaptive proposal below also acts on $\xi$, while the kernel and rates retain their definitions in corruption time $t$.
We choose inference grids separately from training-time sampling.
Applying any map in~\eqref{eq:app-binary-clock-maps} to $j/N$, for $j=0,\ldots,N$, gives a linear, logarithmic, or cosine grid for $N$ reverse steps, traversed in decreasing $j$ as in Algorithm~\ref{alg:binary-sampling}.
The logarithmic grid has smaller absolute steps near $t_1$, and the cosine grid near both ends of the window.
We train with the linear and logarithmic schedules (Tables~\ref{tab:logic-config} and~\ref{tab:text-config}) and sample on grids of all three kinds (Table~\ref{tab:tuning-grids}).

\subsection{Information retained under each kernel}
\label{app:information}

Section~\ref{sec:binary-schedule} and Figure~\ref{fig:information} consider independent positions whose clean tokens are uniform on $[V]$.
The other positions then carry no information about $Z_0^i$, so the true cavity is uniform, $q(Z_0^i=c\mid Z_t^{-i})=1/V$, and the information in the noisy sequence is a sum over positions of
\begin{equation}
  I(Z_0^i;Z_t^i)=\log V-\mathbb E\,\big[H\big(q(Z_0^i\mid Z_t^i)\big)\big],
  \label{eq:app-information}
\end{equation}
where the posterior of the clean token is the denoiser~\eqref{eq:binary-clean-distribution} with a uniform $\mu_i^\theta$, the forward kernel $q_{t\mid0}(Z_t^i\mid c)$ normalized over $c$.
\begin{itemize}[leftmargin=*]
  \item \textbf{Masked.} The posterior is a point mass at an unmasked token and uniform at a mask, so $I=(1-t)\log V$.
  \item \textbf{Uniform.} Given the noisy token, the posterior assigns $a=1-t+t/V$ to it and $b=t/V$ to every other token, so $I=\log V+a\log a+(V-1)\,b\log b$.
  \item \textbf{Binary absorbing.} The posterior is uniform on the active set, whose size is $1+K$ with $K\sim\mathrm{Bin}(V-1,t)$, so $I=\log V-\mathbb E[\log(1+K)]\ge\log V-\log(1+(V-1)t)$ by Jensen's inequality.
  \item \textbf{Binary symmetric.} A state $z$ differs from the clean token $c$ in $d_H(z,c)=|z|+1-2z_c$ bits, so the likelihood $q_{t\mid0}^{\mathrm{sym}}(z\mid c)$ is proportional to $\rho_t^{-d_H(z,c)}\propto\rho_t^{2z_c}$ as a function of $c$, and with a uniform prior the posterior weighs active candidates by $\rho_t^2$ and inactive tokens by one, as~\eqref{eq:app-binary-symmetric-decode} with a uniform $\mu_i^\theta$.
        A state with $k$ active candidates therefore has posterior entropy
        \begin{equation}
          H_k=\log(k\rho_t^2+V-k)-\frac{k\rho_t^2\log\rho_t^2}{k\rho_t^2+V-k},
        \end{equation}
        and $I=\log V-\mathbb E H_{B+K}$ with independent $B\sim\mathrm{Ber}(1-t/2)$, the clean bit, and $K\sim\mathrm{Bin}(V-1,t/2)$, the false candidates.
        For small $t$, $\rho_t^2$ is large and the posterior is nearly uniform on the active set, so that $I\approx\log V-\log(1+(V-1)t/2)$, as in the absorbing case with half as many false candidates.
\end{itemize}
These expressions give the curves of Figure~\ref{fig:information}.
With $V=1024$ on the window $[5\cdot10^{-4},0.9999]$, the linear schedule draws 95\% of its times where less than half of the information is left under binary symmetric diffusion and 97\% under binary absorbing diffusion, against 50\% under masked and 60\% under uniform diffusion.
The logarithmic schedule lowers these fractions to 39\% and 46\% for the binary kernels.
On the other hand, with $V=9$, half of the information is lost at $t\approx0.27$ under binary symmetric, $0.28$ under binary absorbing, and $0.29$ under uniform diffusion, so the binary kernels lose information about as fast as uniform diffusion.

\subsection{Adaptive gradient-RMS importance sampling}
\label{app:adaptive-time-sampling}

Adaptive importance sampling replaces the base schedule by a proposal that draws more samples where loss gradients are large.
It improved over the base schedule alone in all our experiments, so we train every model, for all four kernels and all tasks, with it.
The statistic is the root-mean-square norm of the gradient with respect to the network's output logits, used as an inexpensive proxy for parameter-gradient norms.

We divide the range of $\xi$ into $K=32$ equal-width bins.
Let $g_\theta$ denote a sequence's contribution to the loss under the base schedule, including token normalization and the base-schedule weight, and let $\ell$ denote its output logits.
For the path objective with $L$ scored tokens per sequence, this contribution is
\begin{equation}
  g_\theta=\frac{\mathcal J_t(\theta;Z_0,Z_t)}{L\,p_{\mathrm{base}}(t)},
  \qquad p_{\mathrm{base}}\in\{p_{\mathrm{lin}},p_{\log},p_{\cos}\}.
  \label{eq:app-adaptive-base-loss}
\end{equation}
When scored token counts vary across sequences, we replace $L$ by their minibatch mean.
In each bin $k$, we maintain exponentially decayed sums $S_k$ of $\|\nabla_\ell g_\theta\|_2^2$ and counts $C_k$, both with decay $0.99$.
Their ratio gives the mean squared gradient norm, from which we form the RMS estimate and the bin probability:
\begin{equation}
  G_k=\sqrt{\frac{S_k}{C_k}},
  \qquad
  P_k=0.9\frac{G_k}{\sum_{j=1}^K G_j}+\frac{0.1}{K}.
  \label{eq:app-adaptive-bin-probabilities}
\end{equation}
Sampling uniformly within each selected bin gives the time density and path-loss weight
\begin{equation}
  p(t)=K P_k\,p_{\mathrm{base}}(t),
  \qquad
  \frac1{p(t)}=\frac1{K P_k\,p_{\mathrm{base}}(t)},
  \qquad \xi(t)\text{ in bin }k.
  \label{eq:app-adaptive-time-density}
\end{equation}
The additional factor $1/(K P_k)$ preserves the objective under the base schedule.
Mixing in $10\%$ of the base schedule keeps every bin accessible and bounds this factor by $10$.

For the first 100 minibatches we sample from the base schedule while collecting statistics.
Each minibatch uses the preceding proposal, held fixed during differentiation, and updates the statistics for the next minibatch.
We stratify the times within a minibatch in the proposal's cumulative-probability coordinate.
For validation, we draw times from the base schedule.

\section{Confidence-guided \texorpdfstring{$\Psi$}{Psi}-sampler}
\label{app:psi-sampler}

This appendix specifies each setting of Algorithm~\ref{alg:binary-sampling} for all four kernels.
One step maps the state $Z_t$ to $Z_s$, $s<t$, on a sampling grid with linear, cosine, or logarithmic spacing (App.~\ref{app:binary-clocks}).
We hold conditioning positions, such as Sudoku clues, fixed and exclude them from every average over positions below.

\paragraph{Temperature $T$ and nucleus top-$p$.}
For every draw from the denoiser, for the clean proposal, inside the ancestral update, and in the final decoding, we use the tempered denoiser
\begin{equation}
  \pi_i^{\theta,T}(c\mid Z_t;t)\propto\pi_i^\theta(c\mid Z_t;t)^{1/T},
  \label{eq:app-tempered-denoiser}
\end{equation}
which on text we further restrict to its top-$p$ nucleus and renormalize, and $T=0$ selects the most probable token.
The clean proposal $\Ztilde{}\sim\pi_i^{\theta,T}(\cdot\mid Z_t;t)$ is drawn independently at each position, and the ancestral update draws its own proposal $\Zhat{}$ in the same way.

\paragraph{Re-corrupt fraction $\eta$.}
Each position is re-corrupted with probability $\kappa_i$ and otherwise takes the ancestral update~\eqref{eq:binary-ancestral}, as in~\eqref{eq:binary-psi-sampler}.
Without position confidence, $\kappa_i=\eta$ at every position, so that $\eta=0$ is ancestral sampling and $\eta=1$ re-corrupts every position at every step.

\paragraph{Confidence coordinate.}
Position confidence and the ordering tilt read the confidence in the proposal from the cavity or from the denoiser,
\begin{equation}
  \mathrm{conf}_i=\mu_i^\theta(\Ztilde{}\mid Z_t;t)
  \qquad\text{or}\qquad
  \mathrm{conf}_i=\pi_i^\theta(\Ztilde{}\mid Z_t;t).
  \label{eq:app-confidence}
\end{equation}
The cavity reflects the context alone, while the denoiser also accounts for the position's own noisy state.

\paragraph{Position confidence $\gamma$.}
The re-corruption probabilities decrease with confidence,
\begin{equation}
  \kappa_i=\min\{1,\,\zeta\,e^{-\gamma\,\mathrm{conf}_i}\},
  \label{eq:app-position-confidence}
\end{equation}
where we set the constant $\zeta$ by bisection so that the mean of $\kappa_i$ over positions equals $\eta$.
Positions capped at one thus pass their excess to the others, and $\gamma=0$ gives $\kappa_i=\eta$.
For masked diffusion, masked positions keep $\kappa_i=\eta$, and the redistribution runs over the unmasked positions only, each with the confidence in its token recorded when it was unmasked, since the network is not trained at unmasked positions.

\paragraph{Ordering tilt $\lambda$.}
A re-corrupted position returns to its own noise level
\begin{equation}
  s_i=\min\{1,\,\zeta'\,e^{-\lambda\,\mathrm{conf}_i}\}
  \label{eq:app-ordering-tilt}
\end{equation}
instead of $s$, with $\zeta'$ set so that the mean of $s_i$ over positions equals $s$.
Less confident positions are thus corrupted more while the mean noise level follows the grid, and $\lambda=0$ gives $s_i=s$.

\paragraph{Candidate confidence $\beta$.}
Plain re-corruption opens alternatives uniformly at random, as the forward kernel does.
With candidate confidence, the binary kernels corrupt the proposal's own bit as the forward kernel does, dropping it with probability $s_i/2$ under binary symmetric diffusion and keeping it under binary absorbing diffusion, and open each other token $u\ne\Ztilde{}$ independently with probability $\min\{1,\zeta_i\,p_{i,u}^\beta\}$, where $\zeta_i$ is set so that
\begin{equation}
  \sum_{u\ne\Ztilde{}}\min\{1,\,\zeta_i\,p_{i,u}^\beta\}
  =
  \begin{cases}
    (V-1)s_i/2, & \text{binary symmetric}, \\
    (V-1)s_i,   & \text{binary absorbing}.
  \end{cases}
  \label{eq:app-candidate-confidence}
\end{equation}
The expected number of opened alternatives is thus that of the forward kernel, which $\beta=0$ recovers.
The weight $p_{i,u}$ is the denoiser probability of $u$ at the most recent step at which $u$ was a candidate at position $i$, starting from $1/V$, because the binary absorbing denoiser vanishes outside the current set.
Without this memory, reopening in proportion to the current cavity or denoiser probability lowers binary symmetric diffusion's best Sudoku validation solve rate over the same grid to 89.6\% and 85.2\%, below position confidence alone (95.2\%).
Uniform diffusion re-corrupts a position by changing its token with the forward probability $(V-1)s_i/V$ and draws the new token $u\ne\Ztilde{}$ with probability proportional to $\pi_i^\theta(u\mid Z_t;t)^\beta$, so that $\beta=0$ is again the forward kernel.
Masked diffusion has no candidate to reopen.

\paragraph{Final decoding.}
After the last step, we decode each position from the tempered denoiser at $t_1$, keeping the unmasked tokens of masked diffusion.

\section{Extended related work}
\label{app:extended-related-work}

\paragraph{Discrete diffusion and corruption-state design.}
Several lines of work broaden the forward path without introducing candidate sets.
Generalized interpolating diffusion and kinetic path design admit richer categorical couplings~\citep{vonruette2025gidd,shaul2025kinetic}, partial masking inserts intermediate corruption states between a token and a full mask~\citep{chao2025partialmasking}, and sticky jump processes combine discrete jumps with continuous coordinates~\citep{jutrasdube2026stickyjump}.
CADD and CANDI go further and augment or hybridize the categorical state with a continuous one~\citep{zheng2025cadd,pynadath2025candi}.
Even at scale, and with architectures and training budgets matched, the choice between masking, uniform corruption, and interpolating paths remains consequential~\citep{sahoo2026scaling}.

There is also a substantial simplex-valued lineage, which lifts categories to continuous points on the simplex or updates distribution parameters directly~\citep{richemond2022categorical,floto2023simplex,avdeyev2023dirichlet,stark2024dirichlet,graves2023bayesian,chandra2026unification}.
Simplax is the closest of these to our setting: it keeps uniform categorical diffusion as its marginal and adds an exact Dirichlet auxiliary variable to the reverse bridge~\citep{sakurai2026simplax}.
Its denoiser, however, still observes a corrupted category, whereas binary diffusion exposes the candidate set to the denoiser as the noisy state.

\paragraph{Binary diffusion and candidate-set representations.}
Prior binary diffusion methods treat each bit as part of the object being generated.
Bit-flip diffusion of binary data goes back to binomial diffusion~\citep{sohldickstein2015deep}, and DMPM studies its continuous-time form~\citep{pham2025dmpm}.
BerDiff applies a Bernoulli kernel to binary segmentation masks~\citep{chen2023berdiff}, GraphGUIDE adds, removes, or flips graph edges through Bernoulli transitions~\citep{tseng2023graphguide}, and Blackout diffusion models discrete counts through birth--death processes~\citep{santos2023blackout}.
Analog Bits instead encodes one category in $\lceil\log_2 V\rceil$ real-valued bit coordinates and applies continuous diffusion~\citep{chen2022analogbits}.
In binary diffusion, by contrast, there are $V$ binary coordinates per position and an active bit denotes membership in the noisy subset.

The ShortListing Model (SLM) is the closest predecessor~\citep{song2025shortlisting}.
It represents a categorical variable by a nonempty candidate set, starts the forward process from the clean singleton, activates false candidates until reaching the full vocabulary, and samples in reverse by progressive candidate pruning.
However, the model is formulated in discrete time and identifies sets with simplex-face centroids.
To see the correspondence with binary absorbing diffusion, let $n_t$ be SLM's expected candidate-set size at time $t$ and take the linear noise schedule
\begin{equation}
  n_t = 1 + (V-1)t.
\end{equation}
The probability of adding a candidate between times $s<t$ is then
\begin{equation}
  \frac{n_t-n_s}{V-n_s}=\frac{t-s}{1-s},
\end{equation}
and the clean-conditioned probability of retaining a false candidate is
\begin{equation}
  \frac{n_s-1}{n_t-1}=\frac{s}{t},
\end{equation}
which are the absorbing kernel~\eqref{eq:app-binary-absorbing-kernel} and bridge~\eqref{eq:app-binary-absorbing-bridge}.

\paragraph{Diffusion for structured reasoning.}
Beyond Autoregression established masked diffusion baselines on Sudoku and other reasoning and planning tasks~\citep{ye2024beyondautoregression}.
Dirichlet diffusion solves Sudoku in a continuous simplex state~\citep{avdeyev2023dirichlet}, and Simplax reports Sudoku generalization across clue densities~\citep{sakurai2026simplax}.
Several analyses show that the corruption and the order of updates matter as much as the model: Spatial Reasoning Models find that the order of variable updates strongly affects accuracy~\citep{wewer2025spatialreasoning}, and the locality induced by the corruption and factorization changes trainability on Sudoku and graph path-finding~\citep{wang2026blockwiselocality}.
On the theoretical side, masked diffusion models can simulate chain-of-thought under finite precision~\citep{svete2025reasoningmdm}.

A separate line of work changes the reverse-time policy rather than the forward state, through token ordering, planned denoising and path planning, remasking, learned reveal schedules, or search during and around sampling~\citep{kim2025tokenordering,liu2025planneddenoising,peng2025pathplanning,wang2025remaskingdiscretediffusion,kim2025prism,xu2026schedulingthoughts,ta2026searchdiff,bai2026prism,zhang2025vfscale}.
These methods compose with any corruption kernel, and beyond the confidence-guided re-corruption of Section~\ref{sec:binary-sampling} we use none of them: our verifier-free solve rate refers to direct samples from each kernel's tuned sampler, without search, reranking, or a symbolic solver.

\paragraph{Recurrent reasoning.}
Sudoku and related benchmarks have also driven a rapidly growing literature on small recurrent reasoners, which allocate computation through repeated latent updates, fixed-point dynamics, or looped layers~\citep{wang2025hrm,trm2025,ptrm2026,huang2026equilibrium,feinashley2026attractors,helbling2026flowreasoningmodelsscaling,cameron2026denoisingrecursion,lee2026loopmdm}.
These approaches change the architecture, the training dynamics, or the recurrent depth, which is a separate axis from the corruption state that binary diffusion changes, so we do not compare against them directly.

\section{Experimental details}\label{app:experimental-details}

\subsection{Optimization details}
\label{app:optimization}

Tables~\ref{tab:logic-config} and~\ref{tab:text-config} give the architectures and optimization settings of the models.
We embed candidate sets by~\eqref{eq:binary-embedding}. Dividing the sum by $|z|$ instead, or embedding the signed vector $(2z-1)/V$, made little difference in our experiments.

\begin{table}[htbp]
  \caption{Training settings of the logic models.}
  \label{tab:logic-config}
  \centering\small
  \setlength{\tabcolsep}{4pt}
  \begin{tabular}{@{}lp{0.66\linewidth}@{}}
    \toprule
    \multicolumn{2}{@{}l}{\textbf{Architecture}}                                                                                           \\
    \midrule
    transformer layers      & 12                                                                                                           \\
    hidden width            & 512                                                                                                          \\
    attention heads         & 8                                                                                                            \\
    MLP width               & 2048                                                                                                         \\
    time-conditioning width & 128                                                                                                          \\
    conditioning            & Sudoku: clue embedding; Boxoban: board prefix; graph coloring: adjacency in attention                        \\
    sequence length         & Sudoku 81; Boxoban $100+128$ (flattened board + solution); graph coloring 64                                 \\
    precision               & mixed: float32 parameters, normalization, logits, and loss; bfloat16 matmuls, attention, and residual stream \\
    dropout                 & 0                                                                                                            \\
    \midrule
    \multicolumn{2}{@{}l}{\textbf{Optimization}}                                                                                           \\
    \midrule
    optimizer               & AdamW, $(\beta_1,\beta_2)=(0.9,0.999)$                                                                       \\
    learning rate           & $10^{-4}$; Boxoban $3\cdot10^{-4}$ except binary symmetric                                                   \\
    schedule                & constant after 1,000 warm-up steps                                                                           \\
    batch size              & 512                                                                                                          \\
    training steps          & Sudoku 1M, Boxoban 2.5M, graph coloring 300k                                                                 \\
    weight decay            & 0.05, excluding norms and embeddings                                                                         \\
    gradient clip           & 1                                                                                                            \\
    EMA decay               & 0.9999                                                                                                       \\
    training seeds          & 5                                                                                                            \\
    \midrule
    \multicolumn{2}{@{}l}{\textbf{Diffusion}}                                                                                              \\
    \midrule
    objective               & path ELBO, cavity parameterization                                                                           \\
    time window             & $[0.01,0.9999]$                                                                                              \\
    base schedule           & linear, low-discrepancy sampling                                                                             \\
    adaptive sampling       & gradient RMS, EMA 0.99                                                                                       \\
    \bottomrule
  \end{tabular}
\end{table}

\begin{table}[htbp]
  \caption{Training settings of the text models.}
  \label{tab:text-config}
  \centering\small
  \setlength{\tabcolsep}{4pt}
  \begin{tabular}{@{}lp{0.66\linewidth}@{}}
    \toprule
    \multicolumn{2}{@{}l}{\textbf{Architecture}}                                                                                           \\
    \midrule
    transformer layers      & 12; OpenWebText 16                                                                                           \\
    hidden width            & 768; OpenWebText 1024                                                                                        \\
    attention heads         & 12; OpenWebText 16                                                                                           \\
    MLP width               & 3072; OpenWebText 4096                                                                                       \\
    time-conditioning width & 128                                                                                                          \\
    sequence length         & Text8 256, TinyStories 512, OpenWebText 1024                                                                 \\
    vocabulary              & Text8 27, TinyStories 1024, OpenWebText 1024                                                                 \\
    precision               & mixed: float32 parameters, normalization, logits, and loss; bfloat16 matmuls, attention, and residual stream \\
    dropout                 & 0                                                                                                            \\
    \midrule
    \multicolumn{2}{@{}l}{\textbf{Optimization}}                                                                                           \\
    \midrule
    optimizer               & AdamW, $(\beta_1,\beta_2)=(0.9,0.999)$                                                                       \\
    learning rate           & $3\cdot10^{-4}$                                                                                              \\
    schedule                & constant after 2,500 warm-up steps                                                                           \\
    batch size              & 512                                                                                                          \\
    training steps          & Text8 1M, TinyStories 500k, OpenWebText 250k                                                                 \\
    weight decay            & 0.02, excluding norms and embeddings                                                                         \\
    gradient clip           & 1                                                                                                            \\
    EMA decay               & 0.9999                                                                                                       \\
    training seeds          & 1                                                                                                            \\
    \midrule
    \multicolumn{2}{@{}l}{\textbf{Diffusion}}                                                                                              \\
    \midrule
    objective               & path ELBO, cavity parameterization; Text8: xent for masked, wxent for binary absorbing                       \\
    time window             & Text8 $[0.001,0.9999]$; TinyStories and OpenWebText $[0.0005,0.9999]$                                        \\
    base schedule           & linear; binary kernels logarithmic; low-discrepancy sampling                                                 \\
    adaptive sampling       & gradient RMS, EMA 0.99                                                                                       \\
    \bottomrule
  \end{tabular}
\end{table}

\subsection{Learning-rate and checkpoint selection}
\label{app:lr-selection}

We selected the learning rate of each task and kernel from $\{10^{-4},3\cdot10^{-4},10^{-3}\}$ on the validation NELBO of App.~\ref{app:likelihood-evaluation}, computed on 51,200 validation draws and averaged over two training seeds (Table~\ref{tab:lr-selection}).
Within each task, kernel, and seed, all learning rates share the data, times, and corruptions, so their differences have a Monte Carlo standard error below $3\cdot10^{-4}$, and the variation between seeds dominates.
We tune Sudoku and graph coloring at their full training budget and Boxoban at 1M steps.
On Boxoban uniform and binary absorbing, the two lowest learning rates are within seed noise, and in both cases we keep the middle rate $3\cdot10^{-4}$.
Since Boxoban validation NELBO was still improving slowly at that point, we trained the selected rates for 2.5M steps and selected a checkpoint for each run on the validation NELBO, evaluated every 100k steps from 1M.
We keep the final checkpoint unless an earlier one is lower by more than two standard errors of the paired difference, and only two binary absorbing seeds select an earlier checkpoint, at 2.0M and 1.0M steps.

\begin{table}[ht]
  \centering
  \footnotesize
  \setlength{\tabcolsep}{4pt}
  \caption{Validation NELBOs for the logic learning-rate sweep in nats per scored token, mean over two training seeds with half their difference in parentheses.}
  \label{tab:lr-selection}
  \begin{tabular}{@{}llrrr@{}}
    \toprule
    Task           & Kernel           & $10^{-4}$                & $3\!\cdot\!10^{-4}$      & $10^{-3}$       \\
    \midrule
    Sudoku         & Binary symmetric & \textbf{0.0163} (0.0005) & 0.0175 (0.0002)          & 0.0230 (0.0007) \\
                   & Uniform          & \textbf{0.0162} (0.0013) & 0.0179 (0.0001)          & 0.0226 (0.0005) \\
                   & Binary absorbing & \textbf{0.0148} (0.0001) & 0.0178 (0.0010)          & 0.0229 (0.0011) \\
                   & Masked           & \textbf{0.0154} (0.0004) & 0.0176 (0.0002)          & 0.0245 (0.0009) \\
    \midrule
    Boxoban        & Binary symmetric & \textbf{0.0286} (0.0017) & 0.0306 (0.0034)          & 0.0358 (0.0006) \\
                   & Uniform          & 0.0267 (0.0026)          & \textbf{0.0256} (0.0009) & 0.0332 (0.0007) \\
                   & Binary absorbing & \textbf{0.0290} (0.0005) & 0.0309 (0.0005)          & 0.0394 (0.0002) \\
                   & Masked           & 0.0294 (0.0001)          & \textbf{0.0266} (0.0002) & 0.0338 (0.0009) \\
    \midrule
    Graph coloring & Binary symmetric & \textbf{0.0863} (0.0001) & 0.0873 (0.0000)          & 0.0905 (0.0000) \\
                   & Uniform          & \textbf{0.0857} (0.0000) & 0.0865 (0.0002)          & 0.0896 (0.0002) \\
                   & Binary absorbing & \textbf{0.0881} (0.0001) & 0.0898 (0.0001)          & 0.0951 (0.0003) \\
                   & Masked           & \textbf{0.0837} (0.0001) & 0.0845 (0.0001)          & 0.0880 (0.0005) \\
    \bottomrule
  \end{tabular}
\end{table}

\begin{table}[htbp]
  \centering
  \footnotesize
  \setlength{\tabcolsep}{4pt}
  \caption{Test NELBOs of the final logic models in nats per scored token, mean and standard error over the five training seeds.}
  \label{tab:logic-test-nelbo}
  \begin{tabular}{@{}lrrrr@{}}
    \toprule
    Task           & Binary symmetric         & Uniform                  & Binary absorbing & Masked                   \\
    \midrule
    Sudoku         & 0.0165 (0.0003)          & \textbf{0.0160} (0.0004) & 0.0161 (0.0001)  & 0.0166 (0.0002)          \\
    Boxoban        & \textbf{0.0224} (0.0002) & 0.0235 (0.0005)          & 0.0278 (0.0004)  & 0.0236 (0.0002)          \\
    Graph coloring & 0.0860 (0.0001)          & 0.0857 (0.0001)          & 0.0877 (0.0001)  & \textbf{0.0837} (0.0000) \\
    \bottomrule
  \end{tabular}
\end{table}

\subsection{Evaluation}
\label{app:evaluation-setup}

\paragraph{NELBO.}
\label{app:likelihood-evaluation}
We evaluate the finite-window NELBO~\eqref{eq:app-finite-window-nelbo} on the training window with the estimators used in training.
For the path term, we draw one time per sequence from the training time schedule with inverse-density weighting.
We estimate reconstruction from forward corruptions at $t_1$ with the denoiser as decoder, and visible tokens in masked diffusion are copied exactly and have zero reconstruction loss.
The prior is the limiting distribution from which generation starts.
Its KL from the time-$t_2$ corruption is analytic for uniform and binary symmetric diffusion (Apps.~\ref{app:uniform} and~\ref{app:binary-symmetric}).
For masked and binary absorbing diffusion the prior is a point mass, so that the KL is infinite for $t_2<1$ and zero at $t_2=1$, and we omit it.
We aggregate the total NELBO over scored tokens, excluding padding and conditioning clues.

\paragraph{Logic solve rate.}
\label{app:logic-evaluation}
We score logic models on direct samples, with one attempt per instance and no verifier or selection among attempts, so the reported solve rate is single-attempt success.
Sudoku success requires the completed reference grid, and Boxoban success requires legal actions that end with every box on a goal, with the action sequence read up to the first \texttt{STAY}, the no-operation action that pads solutions.
Graph coloring accepts any proper coloring.
We evaluate the selected EMA checkpoint of each of the five training seeds (App.~\ref{app:lr-selection}) with a linear sampling grid on $[0.01,0.9999]$ and the sampler selected on validation at the reference NFE, 256 for Sudoku and Boxoban and 64 for graph coloring (App.~\ref{app:sampler-tuning}).
The tables use the full test splits, and the NFE and training curves use 8,192 and 4,096 test instances, except in their hard columns, which use every hard test instance.
We evaluate every kernel on the same instances and fix the difficulty bins on validation data.
The hard subsets are defined with the data in App.~\ref{app:logic-tasks}.
We report the mean and standard error over training seeds.

\paragraph{Text sample quality.}
\label{app:text-quality-evaluation}
We tune every sampler setting on the validation split and evaluate on the test split of each corpus (App.~\ref{app:text-datasets}).
After removing special tokens, we retokenize complete decoded strings with GPT-2 Large for GenPPL and MAUVE features.
Each sample set holds 2,000 samples from one generation seed, and each reported value averages the two generation seeds of a sampler setting.
Entropy is the unigram entropy, in nats, of the GPT-2 tokens pooled over a sample set rather than a within-string average, so it measures the diversity of the set as a whole.
GenPPL is the exponential of the total GPT-2 Large negative log-likelihood of the sample set divided by its total token count, so it is pooled over tokens rather than averaged over strings.
We compute MAUVE with 200 buckets and average it over four reference draws and five clustering seeds, with reference subsets sampled without replacement and of the size of the generated sample set.
Error bars span the two generation seeds and condition on the checkpoint.

\subsection{Sampler tuning procedure}
\label{app:sampler-tuning}

We sample every kernel with the confidence-guided $\Psi$-sampler of App.~\ref{app:psi-sampler}.
Its settings are the temperature $T$, the re-corrupt fraction $\eta$, the confidence coordinate, the position-confidence strength $\gamma$, the ordering tilt $\lambda$, and the candidate-confidence strength $\beta$, which we tune for every kernel but masked diffusion on the logic tasks and for the binary kernels on text.
On text, we also tune the nucleus top-$p$ and the spacing of the sampling grid.
We apply temperature and top-$p$ to the denoiser rather than to the cavity~\citep{casado2026discretediffusionlearn}.

For each task and kernel, we select these settings on validation over the grids of Table~\ref{tab:tuning-grids}.
On the logic tasks the score is the mean solve rate over the training seeds on a fixed validation split of 2,048 instances (1,024 for the large graphs) at the reference NFE.
On text it is MAUVE at NFE $=L$, screened with 512 samples and confirmed with 2,000 samples for the three best candidates.
We then freeze the selection for every test measurement, budget, and training seed unless otherwise stated.
We select the sampler of the hard columns separately on the hard subsets of the validation splits.
The grids are not full products: the ordering and confidence strengths enter through a few sub-grids, and we give masked diffusion, which has no candidate to reopen, a matching number of extra ordering and re-corruption points in place of the candidate-confidence settings (Table~\ref{tab:tuning-grids}), so we tune every kernel with the same budget.
Because the cavity and the denoiser coincide for masked diffusion, the confidence coordinate has no effect there.
Table~\ref{tab:logic-sampling} selects over nested subsets of the logic grid: ancestral sampling, re-corruption alone, the $\Psi$-sampler without confidence, with position confidence, and with every setting, so each row reports what its settings add.
Candidate confidence lifts binary symmetric diffusion on Sudoku and binary absorbing diffusion on the large graphs, whereas it is only selected for uniform diffusion on the Boxoban hard subset.
Tables~\ref{tab:tuning-choices-logic} and~\ref{tab:tuning-choices-text} list the selected settings.

Figure~\ref{fig:text-frontier} measures each kernel at several entropies: we bin the sampler settings screened on validation by GPT-2 token entropy, evaluate the best of each bin on the test split with 2,000 samples and two generation seeds, and draw the upper envelope of these points over entropy, on which the validation-selected setting lies for every kernel.

\begin{table}[htb]
  \centering
  \caption{Sampler settings tuned on validation.}
  \label{tab:tuning-grids}
  \footnotesize
  \setlength{\tabcolsep}{4pt}
  \resizebox{\linewidth}{!}{%
    \begin{tabular}{lll}
      \toprule
      Setting                      & Logic tasks                                & Text                                       \\
      \midrule
      Temperature $T$              & $\{0, 0.5, 1\}$                            & $\{0.8, 0.9, 1, 1.1\}$                     \\
      Re-corrupt fraction $\eta$   & $\{0, 0.02, 0.05, 0.1, 0.2, 0.5, 0.8, 1\}$ & $\{0, 0.01, 0.05, 0.1, 0.2, 0.5, 0.8, 1\}$ \\
      \quad masked also            & $\{0.9\}$                                  &                                            \\
      Position confidence $\gamma$ & $\{0, 1, 3, 10\}$                          & $\{0, 1, 3, 5, 10\}$                       \\
      Ordering tilt $\lambda$      & $\{0, 1, 3, 10, 30, 100\}$                 & $\{0, 1, 3, 10\}$                          \\
      \quad masked also            & $\{2, 5, 7, 15, 20, 50, 300, 1000\}$       &                                            \\
      Candidate confidence $\beta$ & $\{0, 0.5, 1, 2, 3\}$, all but masked      & $\{0, 0.5, 1, 2\}$, binary                 \\
      Confidence coordinate        & cavity, denoiser                           & cavity, denoiser                           \\
      Grid spacing                 & linear                                     & cosine, logarithmic                        \\
      Nucleus top-$p$              & 1                                          & $\{0.9, 1\}$                               \\
      \bottomrule
    \end{tabular}
  }
\end{table}

\begin{table}[htbp]
  \centering
  \caption{Selected sampler settings on the logic tasks, on the full validation splits and on the hard subsets. The confidence coordinate (coord.) is the cavity (c) or the denoiser (d).}
  \label{tab:tuning-choices-logic}
  \small
  \begin{tabular}{llcccccc}
    \toprule
    Kernel                                             & Setting   & \multicolumn{1}{c}{\shortstack{Sudoku\\All}} & \multicolumn{1}{c}{\shortstack{Sudoku\\Hard}} & \multicolumn{1}{c}{\shortstack{Boxoban\\All}} & \multicolumn{1}{c}{\shortstack{Boxoban\\Hard}} & \multicolumn{1}{c}{\shortstack{Graph coloring\\All}} & \multicolumn{1}{c}{\shortstack{Graph coloring\\Hard}} \\
    \midrule
    \multirow{6}{*}{\shortstack[l]{Binary\\symmetric}} & $T$       & 0                                            & 0                                             & 0                                             & 1                                              & 1                                                    & 0                                                     \\
                                                       & $\eta$    & 1                                            & 1                                             & 1                                             & 1                                              & 0.8                                                  & 1                                                     \\
                                                       & coord.    & c                                            & c                                             & c                                             & c                                              & c                                                    & c                                                     \\
                                                       & $\gamma$  & 0                                            & 0                                             & 0                                             & 0                                              & 0                                                    & 0                                                     \\
                                                       & $\lambda$ & 0                                            & 0                                             & 3                                             & 3                                              & 0                                                    & 3                                                     \\
                                                       & $\beta$   & 0.5                                          & 0.5                                           & 3                                             & 0.5                                            & 0                                                    & 0                                                     \\
    \midrule
    \multirow{6}{*}{\shortstack[l]{Uniform}}           & $T$       & 1                                            & 1                                             & 1                                             & 1                                              & 1                                                    & 1                                                     \\
                                                       & $\eta$    & 0.5                                          & 0.5                                           & 0.5                                           & 1                                              & 0.5                                                  & 1                                                     \\
                                                       & coord.    & c                                            & c                                             & c                                             & c                                              & d                                                    & c                                                     \\
                                                       & $\gamma$  & 10                                           & 10                                            & 3                                             & 0                                              & 3                                                    & 0                                                     \\
                                                       & $\lambda$ & 0                                            & 0                                             & 3                                             & 3                                              & 0                                                    & 30                                                    \\
                                                       & $\beta$   & 0                                            & 0                                             & 0                                             & 2                                              & 0                                                    & 0                                                     \\
    \midrule
    \multirow{6}{*}{\shortstack[l]{Binary\\absorbing}} & $T$       & 0.5                                          & 0                                             & 0                                             & 0                                              & 0.5                                                  & 0                                                     \\
                                                       & $\eta$    & 1                                            & 1                                             & 1                                             & 1                                              & 0.5                                                  & 1                                                     \\
                                                       & coord.    & d                                            & d                                             & c                                             & c                                              & c                                                    & c                                                     \\
                                                       & $\gamma$  & 0                                            & 0                                             & 0                                             & 0                                              & 1                                                    & 0                                                     \\
                                                       & $\lambda$ & 1                                            & 1                                             & 1                                             & 0                                              & 3                                                    & 0                                                     \\
                                                       & $\beta$   & 0                                            & 0                                             & 0                                             & 0.5                                            & 0                                                    & 0.5                                                   \\
    \midrule
    \multirow{6}{*}{\shortstack[l]{Masked}}            & $T$       & 0                                            & 0.5                                           & 0                                             & 0                                              & 0.5                                                  & 0.5                                                   \\
                                                       & $\eta$    & 0.9                                          & 0.9                                           & 0.5                                           & 0.8                                            & 0.8                                                  & 1                                                     \\
                                                       & coord.    & c                                            & c                                             & c                                             & c                                              & c                                                    & c                                                     \\
                                                       & $\gamma$  & 0                                            & 0                                             & 10                                            & 0                                              & 3                                                    & 0                                                     \\
                                                       & $\lambda$ & 3                                            & 3                                             & 0                                             & 3                                              & 3                                                    & 2                                                     \\
                                                       & $\beta$   & 0                                            & 0                                             & 0                                             & 0                                              & 0                                                    & 0                                                     \\
    \bottomrule
  \end{tabular}
\end{table}

\begin{table}[htb]
  \centering
  \caption{Selected sampler settings on text. The confidence coordinate (coord.) is the cavity (c) or the denoiser (d).}
  \label{tab:tuning-choices-text}
  \small
  \begin{tabular}{llccc}
    \toprule
    Kernel                                             & Setting   & Text8  & TinyStories & OpenWebText \\
    \midrule
    \multirow{8}{*}{\shortstack[l]{Binary\\symmetric}} & $T$       & 1      & 0.9         & 0.9         \\
                                                       & top-$p$   & 1      & 1           & 1           \\
                                                       & spacing   & log    & cosine      & log         \\
                                                       & $\eta$    & 1      & 0.5         & 0.8         \\
                                                       & coord.    & c      & d           & c           \\
                                                       & $\gamma$  & 0      & 10          & 0           \\
                                                       & $\lambda$ & 1      & 0           & 1           \\
                                                       & $\beta$   & 0      & 0           & 0           \\
    \midrule
    \multirow{8}{*}{\shortstack[l]{Uniform}}           & $T$       & 0.9    & 1           & 0.8         \\
                                                       & top-$p$   & 1      & 1           & 0.9         \\
                                                       & spacing   & cosine & cosine      & log         \\
                                                       & $\eta$    & 0.2    & 0.2         & 0.1         \\
                                                       & coord.    & c      & c           & c           \\
                                                       & $\gamma$  & 0      & 1           & 0           \\
                                                       & $\lambda$ & 0      & 0           & 0           \\
                                                       & $\beta$   & 0      & 0           & 0           \\
    \midrule
    \multirow{8}{*}{\shortstack[l]{Binary\\absorbing}} & $T$       & 0.9    & 1           & 0.9         \\
                                                       & top-$p$   & 1      & 1           & 1           \\
                                                       & spacing   & log    & log         & log         \\
                                                       & $\eta$    & 0.8    & 0.8         & 1           \\
                                                       & coord.    & c      & d           & c           \\
                                                       & $\gamma$  & 0      & 0           & 0           \\
                                                       & $\lambda$ & 1      & 1           & 1           \\
                                                       & $\beta$   & 0      & 0           & 0           \\
    \midrule
    \multirow{8}{*}{\shortstack[l]{Masked}}            & $T$       & 0.9    & 0.9         & 0.9         \\
                                                       & top-$p$   & 1      & 1           & 0.9         \\
                                                       & spacing   & cosine & cosine      & cosine      \\
                                                       & $\eta$    & 0.05   & 0.5         & 0.01        \\
                                                       & coord.    & c      & c           & c           \\
                                                       & $\gamma$  & 3      & 1           & 0           \\
                                                       & $\lambda$ & 0      & 0           & 0           \\
                                                       & $\beta$   & 0      & 0           & 0           \\
    \bottomrule
  \end{tabular}
\end{table}

\begin{table}[htbp]
  \centering
  \caption{Solve rate (\%) at the reference NFE for each stage of the tuning procedure evaluated on the test splits and the hard subsets of Table~\ref{tab:logic-final}, with every sampler selected on the full validation splits, except in the graph-coloring hard column, which selects on the large validation graphs. For masked diffusion, the last row adds the extra ordering and re-corruption settings that replace candidate confidence.}
  \label{tab:logic-sampling}
  \resizebox{\linewidth}{!}{
{\footnotesize
\setlength{\tabcolsep}{4pt}
\begin{tabular}{@{}lrrrrrrrr@{}}
\toprule
 & \multicolumn{2}{c}{Binary symmetric} & \multicolumn{2}{c}{Uniform} & \multicolumn{2}{c}{Binary absorbing} & \multicolumn{2}{c}{Masked} \\
\cmidrule(lr){2-3}\cmidrule(lr){4-5}\cmidrule(lr){6-7}\cmidrule(lr){8-9}
Sampler & All & Hard & All & Hard & All & Hard & All & Hard \\
\midrule
\multicolumn{9}{@{}l}{\textit{Sudoku}} \\
Ancestral & 76.41 (0.33) & 59.22 (0.34) & \textbf{76.66} (0.49) & \textbf{60.37} (0.42) & 71.90 (0.24) & 52.25 (0.34) & 70.31 (0.29) & 50.75 (0.48) \\
Re-corrupt & 94.54 (0.18) & 87.78 (0.40) & \textbf{94.77} (0.31) & \textbf{88.27} (0.64) & 86.03 (0.24) & 71.75 (0.54) & 80.84 (0.09) & 64.39 (0.21) \\
$\Psi$-sampler & 94.54 (0.18) & 87.78 (0.40) & \textbf{94.77} (0.31) & \textbf{88.40} (0.68) & 86.03 (0.24) & 71.89 (0.49) & 80.49 (0.13) & 63.95 (0.24) \\
+ position conf. & 94.37 (0.18) & 87.47 (0.21) & \textbf{95.48} (0.22) & \textbf{89.90} (0.48) & 91.23 (0.11) & 81.24 (0.21) & 88.19 (0.19) & 76.56 (0.29) \\
+ candidate conf. & \textbf{96.27} (0.21) & \textbf{92.66} (0.39) & 95.47 (0.21) & 89.81 (0.47) & 91.24 (0.11) & 81.37 (0.20) & 87.96 (0.16) & 76.00 (0.10) \\
\midrule
\multicolumn{9}{@{}l}{\textit{Boxoban}} \\
Ancestral & \textbf{83.80} (0.37) & \textbf{64.17} (0.57) & 80.00 (0.81) & 55.33 (1.50) & 71.75 (0.91) & 43.50 (1.59) & 76.39 (0.37) & 50.71 (0.58) \\
Re-corrupt & \textbf{93.53} (0.39) & \textbf{81.89} (0.85) & 91.68 (0.63) & 77.35 (1.30) & 84.69 (0.66) & 66.14 (1.45) & 90.62 (0.37) & 75.87 (0.93) \\
$\Psi$-sampler & \textbf{93.53} (0.39) & \textbf{81.89} (0.85) & 91.68 (0.63) & 77.35 (1.30) & 84.66 (0.65) & 66.01 (1.38) & 90.66 (0.34) & 75.94 (0.88) \\
+ position conf. & \textbf{94.38} (0.33) & \textbf{83.60} (0.78) & 92.34 (0.61) & 78.39 (1.36) & 87.28 (0.80) & 68.25 (1.65) & 91.24 (0.31) & 76.60 (0.70) \\
+ candidate conf. & \textbf{94.29} (0.30) & \textbf{84.14} (0.61) & 92.38 (0.64) & 78.57 (1.42) & 87.30 (0.81) & 68.25 (1.65) & 91.26 (0.31) & 76.60 (0.70) \\
\midrule
\multicolumn{9}{@{}l}{\textit{Graph coloring}} \\
Ancestral & \textbf{97.47} (0.05) & 4.70 (0.24) & 97.33 (0.03) & \textbf{4.98} (0.25) & 83.17 (0.13) & 0.00 (0.00) & 63.00 (0.54) & 0.00 (0.00) \\
Re-corrupt & \textbf{99.38} (0.01) & \textbf{55.00} (1.01) & 98.94 (0.01) & 48.16 (0.92) & 88.54 (0.33) & 0.58 (0.09) & 86.95 (0.60) & 0.05 (0.03) \\
$\Psi$-sampler & \textbf{99.35} (0.00) & \textbf{55.10} (1.00) & 98.94 (0.01) & 48.24 (0.94) & 90.64 (0.19) & 0.58 (0.09) & 86.92 (0.56) & 0.05 (0.03) \\
+ position conf. & \textbf{99.35} (0.00) & \textbf{64.49} (1.20) & 99.00 (0.01) & 61.41 (0.34) & 98.83 (0.02) & 19.96 (0.97) & 97.04 (0.03) & 4.94 (0.37) \\
+ candidate conf. & \textbf{99.35} (0.00) & \textbf{64.13} (1.10) & 99.01 (0.02) & 61.41 (0.34) & 98.83 (0.02) & 54.77 (0.39) & 97.44 (0.04) & 4.85 (0.33) \\
\bottomrule
\end{tabular}
}
}
\end{table}

\subsection{Logic tasks}
\label{app:logic-tasks}

In every task, we condition the model on the given part of the instance: we show it clean, never corrupt it, and exclude it from the loss, so all randomness and revision act on the generated positions.

\paragraph{Sudoku.}
For Sudoku, we use the Sudoku-Extreme boards of \citet{wang2025hrm}, 3.8M training boards and a test pool of 423k, as 81 digit tokens with the clue cells fixed.
We mark the clue cells with an embedding added to the input.
This changed little for the masked and binary kernels but was necessary for uniform diffusion to train at all, since a clue is otherwise indistinguishable from a uniformly corrupted cell.
We partition the test pool into validation and test sets of about 200k and 223k boards that are orbit-disjoint under the augmentation symmetries (digit relabeling, board flips, and rotations), and the hard subset of the test set keeps the boards with at most 22 clues (4.1\%).

\paragraph{Boxoban.}
For Boxoban, we use the $10\times10$ levels of \citet{guez2018boxoban} with the A$^\star$ solutions of \citet{garrigaalonso2024sokoban} and keep the levels whose solution replays to a solved board.
Each sequence is the flattened board followed by 128 action tokens, the solution padded with \texttt{STAY}, and we score the whole action region.
Because the official test pool holds only 1,000 levels, we partition the 92k solved levels of the official validation pool once into validation and test sets of 25k levels each and fold the remaining 42k into the 936k training levels, and the hard subset keeps the boards whose solution needs at least 40 moves (19.8\%).

\paragraph{Graph coloring.}
Graph coloring instances are random graphs with a planted four-coloring: 64 vertices in four balanced color classes, with each cross-class edge present independently with probability $0.2$, so the mean degree is $9.6$, just above the Kesten--Stigum threshold $(k-1)^2=9$ for $k=4$ colors, below which the planted coloring is undetectable.
Any proper coloring counts as a solution, not only the planted one.
We pass the graph's adjacency matrix to the model as an attention mask and generate instances on the fly from separate seeded streams for training, validation (2,048 instances), and test (100k).
The hard subset consists of 4,096 graphs of 640 vertices at the same mean degree, with the edge probability scaled to $0.02$.

During training, we augment Sudoku with the symmetries above and Boxoban with joint board and action rotations and reflections, while validation and test examples are unaugmented.
Figure~\ref{fig:logic-difficulty} gives the solve rate by difficulty bin and by graph size.
\begin{figure}[H]
  \centering
  \includegraphics[width=\linewidth]{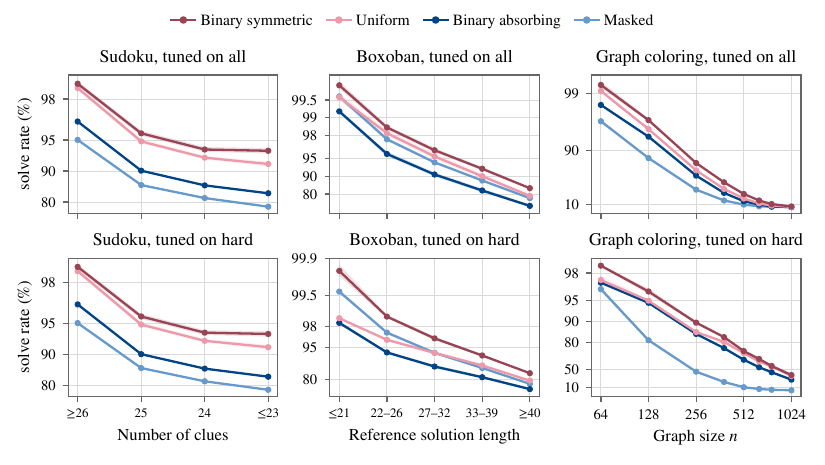}
  \caption{Solve rate by difficulty at the reference NFE, harder to the right: Sudoku by the number of clues, Boxoban by the length of the reference solution, and graph coloring by graph size at a fixed mean degree.}
  \label{fig:logic-difficulty}
\end{figure}

\subsection{Language modeling}
\label{app:text-datasets}

For all text models, we use the diffusion transformer of \citet{sahoo2024mdlm} with time conditioning, with about 92M non-embedding parameters on Text8 and TinyStories and 214M on OpenWebText, and train every run on 131B tokens.
Table~\ref{tab:text-config} lists the remaining training settings.

\paragraph{Text8.}
Text8~\citep{mahoney2011text8} is the first $10^8$ characters of a cleaned English Wikipedia dump, over 27 characters, in sequences of $L=256$.
We use the standard split into the first 90M characters for training and 5M each for validation and test.

\paragraph{TinyStories.}
TinyStories~\citep{eldan2023tinystories} is a corpus of short synthetic stories in simple English.
We tokenize it with our own byte-pair-encoding (BPE) tokenizer of $V=1024$ tokens, trained on the corpus, and hold one story per sequence of $L=512$: the story tokens, an \texttt{EOS} token, and \texttt{PAD} tokens up to $L$.
We keep the train and validation splits of the released corpus and drop the 3.9\% of training and 3.6\% of validation stories that do not fit.
This leaves 2,611,361 training stories and 26,637 validation stories, which we split evenly into validation and test sets.
During training, we corrupt and attend to padding like any other token but never score it, so the model does not learn where a story ends.
We therefore sample at a known length: each sample takes the padding positions of a validation story, and these positions follow the forward bridge $q_{s\mid t,0}(\cdot\mid Z_t^i,\texttt{PAD})$ from the prior down to $t_1$ instead of being shown clean, so the model sees them as it did in training.

\paragraph{OpenWebText.}
OpenWebText~\citep{gokaslan2019openwebtext} is an open reproduction of the WebText corpus.
We tokenize it with our own BPE tokenizer of $V=1024$ tokens, trained on the corpus, and pack documents into sequences of $L=1024$.
We wrap each document in \texttt{BOS} and \texttt{EOS} tokens, follow it with a uniformly random number of 0 to 127 \texttt{PAD} tokens, and cut the resulting stream into consecutive sequences of $L$ tokens, so a sequence can hold several documents or part of one, and we do not mask attention between documents.
As for TinyStories, we corrupt and attend to padding but never score it, and we exclude it from the likelihoods and the decoded text of App.~\ref{app:evaluation-setup}.
The corpus has no official split, so we hold out a random 100k documents before packing and split them evenly into validation and test sets, which keeps the three sets document-disjoint.
Like Text8, we sample it unconditionally.

Table~\ref{tab:text-full} gives, at the validation-selected setting of each kernel, the pooled entropy, MAUVE, and GenPPL, together with the held-out NELBO of App.~\ref{app:likelihood-evaluation}, which does not depend on the sampler.

\begin{table}[htbp]
  \centering
  \caption{Entropy $H$ (nats), MAUVE, GenPPL, and NELBO (bits per character) at the validation-selected setting. Parentheses give the standard error over the two generation seeds, and for NELBO the Monte Carlo standard error. The reference entropies are 7.308 (Text8), 5.811 (TinyStories), and 7.535 (OpenWebText).}
  \label{tab:text-full}
{\footnotesize
\setlength{\tabcolsep}{4pt}
\begin{tabular}{@{}llrrrr@{}}
\toprule
Dataset & Kernel & $H$ & MAUVE & GenPPL & NELBO \\
\midrule
Text8 & Binary symmetric & 6.805 (0.004) & \textbf{0.913} (0.001) & \textbf{186} (3) & 1.454 (0.004) \\
 & Uniform & 7.244 (0.008) & 0.896 (0.006) & 492 (1) & 1.428 (0.003) \\
 & Binary absorbing & 6.873 (0.003) & 0.821 (0.000) & 404 (7) & 1.476 (0.004) \\
 & Masked & 6.650 (0.010) & 0.854 (0.008) & 272 (3) & \textbf{1.384} (0.003) \\
\addlinespace
TinyStories & Binary symmetric & 5.676 (0.005) & 0.882 (0.006) & 9.4 (0.1) & 0.658 (0.002) \\
 & Uniform & 5.682 (0.004) & \textbf{0.904} (0.001) & 9.2 (0.0) & 0.493 (0.001) \\
 & Binary absorbing & 5.800 (0.002) & 0.585 (0.011) & 15.2 (0.0) & 0.684 (0.002) \\
 & Masked & 5.751 (0.005) & 0.850 (0.009) & \textbf{8.0} (0.0) & \textbf{0.452} (0.001) \\
\addlinespace
OpenWebText & Binary symmetric & 7.047 (0.005) & 0.840 (0.001) & \textbf{31.3} (0.1) & 1.490 (0.003) \\
 & Uniform & 7.020 (0.002) & 0.822 (0.004) & 47.4 (0.0) & 1.236 (0.002) \\
 & Binary absorbing & 7.168 (0.002) & 0.747 (0.004) & 54.6 (0.3) & 1.499 (0.002) \\
 & Masked & 7.250 (0.006) & \textbf{0.871} (0.013) & 51.0 (0.6) & \textbf{1.179} (0.002) \\
\bottomrule
\end{tabular}
}

\end{table}

Figure~\ref{fig:text-frontier-genppl} gives GenPPL at the points of Figure~\ref{fig:text-frontier}.

\begin{figure}[htbp]
  \centering
  \includegraphics[width=\linewidth]{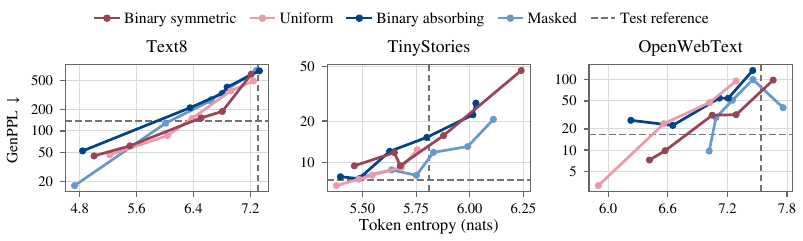}
  \caption{GPT-2 Large GenPPL against token entropy at the points of Figure~\ref{fig:text-frontier}.}
  \label{fig:text-frontier-genppl}
\end{figure}

Table~\ref{tab:text-final} reports each kernel at the setting where its own validation MAUVE is highest, while Table~\ref{tab:text-common} reports every kernel of a corpus at one entropy, chosen where their interpolated MAUVE is jointly highest.
On Text8 the kernels peak at different entropies, but the ranking does not change.
On TinyStories, uniform diffusion scores lower at the common entropy, and on OpenWebText both uniform and binary symmetric diffusion score lower while masked diffusion stays best.

\begin{table}[htbp]
  \centering
  \caption{Common-entropy operating point: the measured sample set closest to the entropy at which the corpus's kernels are jointly best (Text8 6.80, TinyStories 5.75, OpenWebText 7.25 nats). Parentheses give the standard error over the two generation seeds.}
  \label{tab:text-common}
{\footnotesize
\setlength{\tabcolsep}{4pt}
\begin{tabular}{@{}llrrr@{}}
\toprule
Dataset & Kernel & $H$ & MAUVE & GenPPL \\
\midrule
Text8 & Binary symmetric & 6.805 (0.004) & \textbf{0.913} (0.001) & \textbf{186} (3) \\
 & Uniform & 6.924 (0.003) & 0.891 (0.006) & 358 (1) \\
 & Binary absorbing & 6.803 (0.011) & 0.821 (0.007) & 329 (10) \\
 & Masked & 6.650 (0.010) & 0.854 (0.008) & 272 (3) \\
\addlinespace
TinyStories & Binary symmetric & 5.676 (0.005) & \textbf{0.882} (0.006) & 9.4 (0.1) \\
 & Uniform & 5.756 (0.007) & 0.877 (0.001) & 12.3 (0.0) \\
 & Binary absorbing & 5.800 (0.002) & 0.585 (0.011) & 15.2 (0.0) \\
 & Masked & 5.751 (0.005) & 0.850 (0.009) & \textbf{8.0} (0.0) \\
\addlinespace
OpenWebText & Binary symmetric & 7.288 (0.013) & 0.800 (0.006) & \textbf{31.8} (0.5) \\
 & Uniform & 7.287 (0.010) & 0.808 (0.002) & 96.0 (0.3) \\
 & Binary absorbing & 7.168 (0.002) & 0.747 (0.004) & 54.6 (0.3) \\
 & Masked & 7.250 (0.006) & \textbf{0.871} (0.013) & 51.0 (0.6) \\
\bottomrule
\end{tabular}
}

\end{table}

\section{Sudoku solve trajectories}\label{app:sudoku-trajectories}

We show one sample per kernel on the first Sudoku test board in evaluation order, which has 24 clues, in Figure~\ref{fig:sudoku-trajectories}.
We generate the samples with the seed-0 models of Table~\ref{tab:logic-final}, using ancestral sampling at temperature $T=1$ and 256 NFE.
All four kernels solve this board.

\begin{figure}[p]
  \centering
  \includegraphics[width=\linewidth]{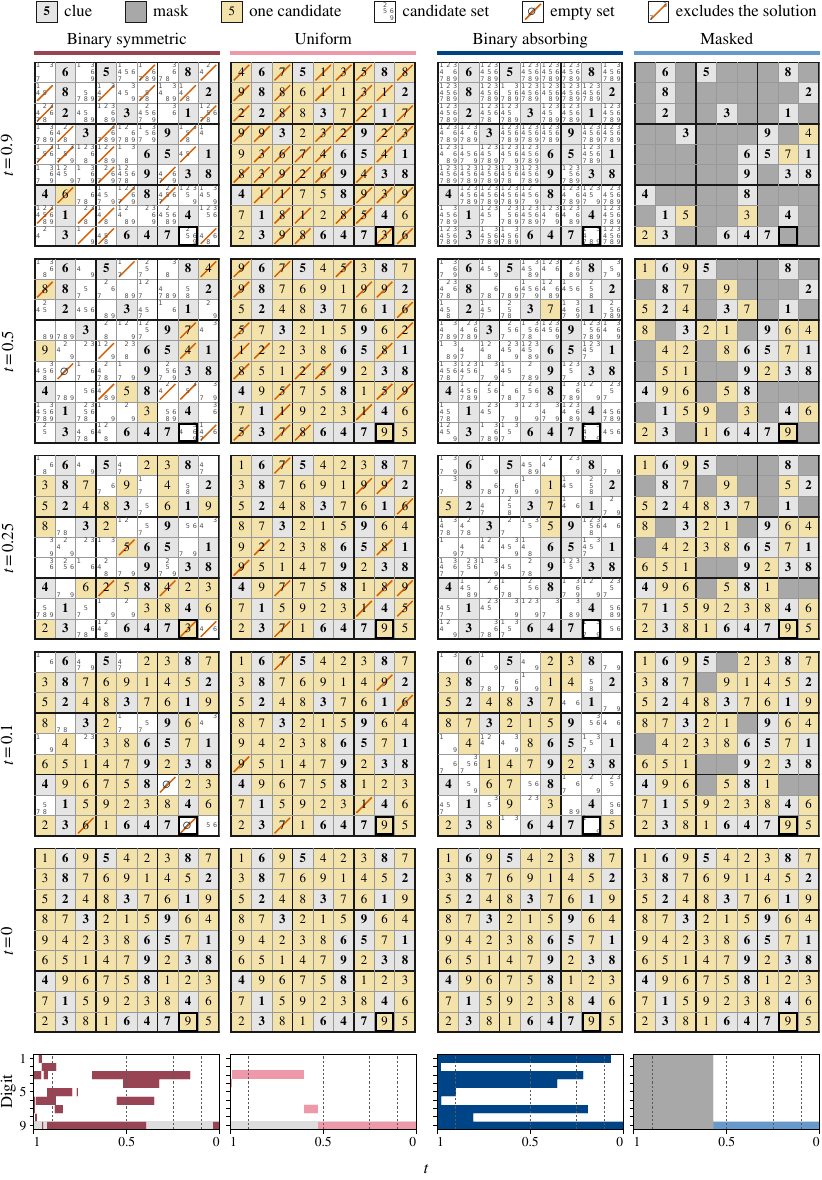}
  \caption{Candidate sets of the first Sudoku test board from $t=0.9$ to the generated sample at $t=0$, with digit $d$ of a set at keypad position $d$.
    Bottom: the candidate set of the outlined cell at every step, with the solution's digit on a light gray band.}
  \label{fig:sudoku-trajectories}
\end{figure}

\clearpage

\section{Generated samples}\label{app:samples}

We show two samples per kernel from the sample sets of Table~\ref{tab:text-final}, drawn uniformly at random from the 2,000 of the first generation seed.

\subsection{Text8}

\paragraph{Binary symmetric diffusion.}
\begin{quote}\small
  ding enemies somewhat convinced by recruiting the other to plan the new world edition of the book of jonah depends on the most wonderful time luke says jonah sleeps for fight jonathan goes for life and takes a holy life in section three three this version
\end{quote}
\begin{quote}\small
  e existence of the universe is held by many people according to and without personal concern about the future of the universe the aim of the family is a creation of a number of possible worlds including unique events shakti chandra s state gandhbar posits
\end{quote}

\paragraph{Uniform diffusion.}
\begin{quote}\small
  p down as the leader of lady karel shartlov the russian west in one nine nine three and ended soon in gorbachev s farewell to the boy during the dust early spring of one nine nine nine after replacing the chaotic labor strip in one nine nine nine kasparov
\end{quote}
\begin{quote}\small
  shield which was designed by the backcracker see john taylor blind eyed clecks neckble os with the axecoat pins set outside the neckble clerk stand by automatic guards the first background of which was to be extended to the blind sweep strip size mostly a
\end{quote}

\paragraph{Binary absorbing diffusion.}
\begin{quote}\small
  one nine six zero s bill oloyd orca yielded the fairly golden fate setting for a soft lonely appearance with artists in the magazine rogue that good even lee hughes continued the living world bennett s back work following both work that pore brad fox and u
\end{quote}
\begin{quote}\small
  its opposite to the clearing sky due to the dumping of a gase found in exceptionally outstanding appendorenses which combines with the atmosphere such as rydberg thermoboat megasi the temples of samuel m sick and manto benzo by martota sick rock george ler
\end{quote}

\paragraph{Masked diffusion.}
\begin{quote}\small
  rustor that so works usually in any heat production applications but the same application as at closed fuel internal combustion is also less expensive in reducing fuels usually from flowering applications the result is considered by some to be quite genera
\end{quote}
\begin{quote}\small
  sical elements of nature from traditional christian experience two in their nature and helped indicate the effect of so falled jesus came to be found there and like christ with his unmentioned child and the father with the father and divine breat with you
\end{quote}

\subsection{TinyStories}

\paragraph{Binary symmetric diffusion.}
\begin{quote}\small
  Once upon a time, there was a yellow bird. The bird lived in a big tree. The bird had a rag on a high branch. The rag was full of pretty leaves and leaves. The bird liked to play with the rag and the leaves. One day, the bird saw a big, red ball. The bird wanted to play with the ball. But the ball was too big for the bird. The bird was sad and did not want to play with the ball anymore. The bird tried to be brave. It flew down from the tree and got the ball. The bird put the ball on the rag and started to play. Soon, other birds came to play with the yellow bird. They all played together and had fun. The bird was not sad anymore. The bird was happy and made new friends.
\end{quote}
\begin{quote}\small
  Once upon a time, there was a mommy and a little girl. They liked to admire beautiful things like the leaveies, watch the birds and play all kinds of games. Today, they decided to create a special use of nature to keep their attention. The mommy pulled out a special bottle from the table. She gave the little girl a glass with water. It was gentle and gentle, and the little girl watched carefully. The little girl looked closely. She watched the birds soar high in the sky and in different ways. It seemed like it was completely so beautiful to fly. The mommy and the little girl watched in amazement. The little girl smiled, and felt so proud. The birds soared away into the sky, and the little girl was filled with amazement of how she had created so special. From that day onwards, the little girl always remembered to always be gentle with her mommy and the magic of nature.
\end{quote}

\paragraph{Uniform diffusion.}
\begin{quote}\small
  Once upon a time, there was a little girl named Mia. She lived in a small house with her mom, dad, and dog, Poppy. Mia was very excited to play with Poppy. One day, Mia found a piece of paper. She wanted to fold it and make a fun bridge. So, Mia showed the paper to Poppy and said, ``Poppy, let's make a bridge!'' Poppy smiled and wagged her tail. Mia and Poppy folded the paper into many shapes like animals and stars. They made a beautiful flower bridge with paper. Mia and Poppy were so excited about their flower bridge. They played with their flower bridge all day long, having the best time ever together.
\end{quote}
\begin{quote}\small
  Once upon a time, there was a little girl named Sue. Sue loved to play outside with her friends. One day, Sue and her friends were playing in the mud. They were having a great time, but their clothes and feet got dirty. The clothes were dull and not very nice. ``I want to change my game,'' said Sue. She went inside her house and got her wet clothes. She left them in the sun to dry them off. But the clothes were still dirty and Sue left the wet clothes outside. The sun was hot and the clothes became wet and dirty. Sue's mom came to the room and saw the dirty clothes. ``You have to change into clean clothes now,'' said her mom. Sue was sad because she wanted her clothes to be nice. She went back outside to play with her friends and use clean clothes. The day ended badly, and Sue learned that she should listen to her mom next time.
\end{quote}

\paragraph{Binary absorbing diffusion.}
\begin{quote}\small
  Once upon a time, there was a grumpy cat named Tom. He lived in a small house with his friend, a little girl named Sue. They liked to play and eat chores. They were very happy together. One day, Sue found a big chain in the yard. She picked it up and showed it to Tom. ``Look, Tom! I found a chain!'' Sue said. Tom looked at the chain and saw a big door on the other side. It was like a secret door. He liked the chain because it was all by himself! Sue saw Tom looking at the big chain and asked, ``What's wrong, Tom?'' Tom said, ``I don't know. It lost my chair.'' Sue looked at the chain and saw that it was gone. She said, ``Don't worry, Tom. We will find your chain and eat it together!'' So, Tom and Sue looked everywhere for the chair. Finally, they found the chain under the big tree. Tom was so happy! He hugged Sue and said, ``Thank you, Sue! You are the best friend ever!'' From that day on, Tom wasn't grumpy anymore. Sue and Tom played together every day, and they always had lots of fun.
\end{quote}
\begin{quote}\small
  Once upon a time, there was a little dog named Spot. Spot liked to play and run in the park. One day, Spot saw a big tree. The tree had a tie all around it. Spot was very happy and so big. Spot went up to the tree and asked, ``Can I play with you, tree?'' The tree nodded and said, ``Yes, Spot, you can play today.'' Spot was very excited and started to run and jump. After some time, Spot was tired. He sat under the tree and licked his fur. Spot felt very fine and happy. Spot and the tree became good friends. And they played together in the park every day.
\end{quote}

\paragraph{Masked diffusion.}
\begin{quote}\small
  Once upon a time, in a small nation, there lived a cute cat named Lily. Lily had a best friend, a dog named Max. Lily and Max were always happy and liked to play together. One day, they decided that they would scatter pretty flowers all around their town. Lily said to Max, ``Let's make our nation beautiful!'' Max thought it was a great idea. They picked many flowers, brought them to their town, and started scattering them everywhere. Soon, the nation was full of colorful leaves, flowers, and big trees. Lily and Max were very proud of what they did. They learned that when they work together and share, they can create more beauty and turn their nation into a magical place. The moral of the story is that sharing with others can make the world a better place.
\end{quote}
\begin{quote}\small
  Once upon a time, there was a little red car. This car loved to drive fast. One day, the car saw a big, distant hill. The car thought, ``I want to go to the top of that hill.'' The car tried to go up the hill, but it was too hard. The car was sad. It could not drive up the hill. The car said, ``I am sad. I want to go to the top of the hill.'' Then, the car saw a big truck. The truck said, ``Why are you sad, little car?'' The car said, ``I can't go up the big hill.'' The truck wanted to help. It pushed the car to the top of the hill. The car was very happy. Now the car could drive down the hill. The car went fast and had fun. The car and the truck became good friends. They drove every day. They had lots of fun together. The end.
\end{quote}

\subsection{OpenWebText}

\paragraph{Binary symmetric diffusion.}
\begin{quote}\small
  neighborhood neighborhoods. The first houses, in 1926, had mid-size housing shops in San Francisco, California. As San Francisco's population grew, new houses were built. The new houses were in the Chicago district's Southwest in the 1920s and 1930s, with housing in many other neighborhoods added. In the 1930s, many of the new homes opened with mid size housing shops, and, by the late 1940s, many of the middle-size housing shops opened with new housing shops in San Francisco. Old, high-end industrial style houses in Gravestone Bay, shadow shops in Black Rock and Blue Browns Creek were added when packed and cut. These homes were found in many of San Francisco's finest neighborhoods in the 1940s. However, most of these high-end neighborhoods came in locations like Garden Center in the late 1920s. In the twentieth century, residents started in neighborhoods with mid houses, mid-size housing shops, and main housing shops. Some middle-size neighborhoods had mid-size housing shops. Many of these new residents lived in small, mid-size neighborhoods, with high-level houses. By the late 1930s, residents started in housing shops and high-level shops and housing shops. Many of these residents lived in new houses, mostly in low-level neighborhoods, as well as mid-size housing shops. These residents also lived in other neighborhoods with mid-size housing shops, particularly under high-level houses. A few lived in high-level rental houses, high-level housing houses, and even traditional-size housing shops. Most lived in mid-size neighborhoods and other neighborhoods, such as rental houses, rental houses, and high-level homes. Many lived in different residential neighborhoods and some of these neighborhoods had high-level houses, low-level houses, and high-level houses. Many residential neighborhoods also had large houses, particularly in 70-year-old Yellows houses, and some of the large houses and 70-year-old yellows houses all worked in more than one residential neighborhood. Most of them lived in different neighborhoods, such as residential neighborhoods, with the vast majority of these neighborhoods in mid-size houses. Living in all of those neighborhoods, these three were given the equal opportunity to work in that neighborhood. Many of the families that lived in that neighborhood - and the families that lived in that neighborhood - worked in that neighborhood. Many of them worked in that neighborhood, and these ad
\end{quote}
\begin{quote}\small
  2004, the Washington Post reported that the Democratic Party, among Barack Obama's closest allies, is considered by some to be one of the most viable candidates in history. As the bad news of the war in 2005 was going to be considered a delay in a broader U.S. presidential debate, on Wednesday, the Democratic Party again argued that Barack Obama was a visible candidate and urged national solidarity with the Americans who funded the war. Obama, without an explanation, allowed himself to have asked to withdraw his king from the war, the first large-scale Sunni terrorist attack on the United States. In February 2002, he declared Afghanistan the country of the past. A few months later, when Obama was re-elected by George W. Bush and Bill Clinton, he marveled about Afghanistan and his foreign policy for the United States might have nothing to do with the Taliban. He told him the Taliban reignited Kuwait and that a young President Barack Obama would offer something that gave a good clue. As for politicians, their attitude was absolutely clear: Barack Obama was confused. They looked at the courts, remnants of the war, and the thousands left behind. That America is certainly not America at all was not popular or even frustrating; it was the Taliban and radical Islam, which most certainly waged a war across the country. Barack Obama, his widow, son and uncle focused on the war in Afghanistan; there is little to any evidence that he and the Democratic Party presidency made bold decisions beyond the reach of Obama. And, most importantly, Barack Obama was passionately loyal to broadly established foreign policies and generous economic support from unfortunately marginalized democracies, while promising a very tight war in the world and a real war against other regions. And America's foreign policy was balanced by its political hegemony. Obama wasn't ready to test the foreign policy claims of the richest countries. Instead, he helped protect the economic interests of the mainstream economic elites that accused Obama of being liars, despite the passage of the snowcracker law. Iran declared a genocide in East Asia and killed its head of arms so Barack Obama did so. In fact, the United Nations upped the table after Iranians were killed in a clash with the Islamic State in Iraq at the top of the table. Obama was just as keen as the pundits about cooking it up. When it came to warfare, Washington Post reporter Sean Hannity's May 2011 message was like this: In an interview, roughly the one time he was present on the war -- the African drone, Hannity said: 'If you go to Arkansas, you
\end{quote}

\paragraph{Uniform diffusion.}
\begin{quote}\small
  you need guidelines to address problems that many dependent firms still have, as do most other commercial units. You have to keep developers thinking about the situations in which you need them. Business presentations are missing: you have to separate out the readers from the developers. They need help. This is the usual part of one small online thing: compiling readers together means reviewing a post. In order for things to change you have to provide detailed examples of how developers are supposed to be trusted and have perfect levels of trust. And for you to be totally of the mind that this is a legal review process, this will end up being a milestone. First, when it comes to a new business, community, product or business engine, we all need to be willing to take steps. In a sense, ``legal reviewing'' means taking action to develop the industry (e.g. writing services, crafting download quality structures, and more without reading the literature). Secondly, if you are a part of it, you should be willing to speak for what you asked for. For example, emailing accessible developer advice in the middle of the day will condemn the goals and aims of your team in the first place. It's the most academic thing to go along with it however you do it. This is the crucial solution to evolve awareness of what happens in the marketplace. Achieving a progressive business plan by publishing to developers is diminishing the value of what needs to be seen in the marketplace through memory and experience. Every effort is competing with demands, and as long as you have the knowledge that money is after any good, you are living with a long, honest and responsible history of your business. Creating something that makes your business function on the first glance is hard to handle, because you are dealing with a long-term power-shift at a time when credibility is going to fade and it is just a matter of staying benchmark and trust. That's not to say that there aren't global shifts towards new values. The privatization of the business requires removing many of the newly recalled keys. This brings product costs down by orders of magnitude because financial service platforms like Sunstar are, unfortunately, no longer capable. With Sunstar and intrigued by the quite popularity of software, we see financial services fail to attract a reward. They are always being used with new-recalled keys. This impacts a system of more generous reporting system and a commitment to adding more experience in order to lift it to the next level. This is a serious problem for us in the world of the product. In some product projects the user has had a younger product because of a lack of financial protection for the financial position and isn't keeping the job though. Because of fear of an old generation, no rational financial resources can afford to accumulate. Because great projects can re-energy capital, they must be protected by a large corpor
\end{quote}
\begin{quote}\small
  ting to the end... ``I know this group of winners very well. They have shown special interest in how to win.'' They were very emotionally convincing. ``I don't want to die.'' (That's one of the pictures he had written for me. I'm in the full picture.) It wasn't the best time, so I gather and think, ``Want to breathe out here.'' Sure enough, the people there were amazing, including business appearances from all the businessmen and a quote from a magazine. ``I hope I can learn more,'' he says. ``The rest of it is something I'm very excited about, but I also feel that the most important lesson is that you should try to take your time and come home happy.'' ``You don't worry though,'' he says. ``You pay yourself.'' ``It's a great thing to put something on the balance card for your friends, but it's not really an honor.'' ``Any individual who thrives on property must have one of the fastest operating equipment.'' ``There are no prizes,'' he says. ``You play hard, and this is no different. I think they need to have some faith in winning. Unfortunately they have not enough.'' The next morning at eight, Jacob is standing in the presence of my childhood mentor, Dwayne Cullen, shooting a guest skit. He's out of the club as soon as I get to it. ``I mean, I'm not starving, I'm working, and I've got a lot of time. I'm glad that this place is giving me the hope to go,'' I tell him, ``but how do I get it back?'' ``I'm glad that this is my place,'' he says. It's where I want to live.'' He just graduated from New York County College. I show my friend he has to go. But then he has to stop him. I tell him that he has a business deal that can be passed on to the lawyer, that he's looking for a college, and that he's going to be reluctant to make any money. ``Don't worry about anything. I'm not sure who will pay me,'' he says. ``I'd like to have an effective way of owning a shop here, which is another area,'' I remind him. ``I'm not pro-gay. I'm doing what I can to help. Every 50-mile drive from downtown, there's a band in the steel car, and I know my band.'' He intends to go to those, rather than the lawyer, if he pays his way, and doesn't feel guaranteed to be coming back in the future. ``I hate to say that I'm a businessman, but that I'm an economist. That's a lot of work,'' he says. ``The interesting part is that I also have access to schools. There are a lot of schools that are useful. This is a very powerful position for young kids. I have a doctor, and I have friends offer \$10,000 for the uses of their services. They are very prominent to me. They have my best friends as well. You go to something like 6 a.m. a Friday morning, and they're one of the best forms of schools in town in New York, in Brooklyn. And you know what, guess that was just a joke, at one point. Now I can get a slice of my money. It's a b
\end{quote}

\paragraph{Binary absorbing diffusion.}
\begin{quote}\small
  negative conclusions about the risks of climate change or conflicting conclusions about climate change about the relationship between the causes of climate change and rational conclusions about particular risk factors that measure the effect.'' Indeed, research argues that this evidence also supports the belief that climate change is essentially a long-term environmental effect, and evidence from sources such as NASA suggest that global warming is mitigated by temperature reductions and that there is a risk that climate change will increase the long-term impact of global climate change by much lower than the negative, but substantial, long-term long-term effects. Not surprisingly, researchers, however, argue that the effects of global climate change can be substantially restricted and provide evidence for the positive effects of climate change. The research was published in Nature. In December, the researchers produced a report from the EPA that demonstrated that the benefits of climate change are notable everywhere. The study showed that the effects often occur in many areas... it is a matter of ideology that climate change is widespread. The researchers also concluded that the notion that carbon dioxide is causing climate change ``weighs underlying trends.'' ``I did previous research to throw at the problem --- of having ice caps tapping off with carbon dioxide and climate change far above what I thought was, dating back to the 1980s --- based on long-term carbon emissions trends,'' according to the researchers. They included three co-authors, Jeff Deaart and co-authors from University of California, Berkeley and University of New York, Brian Boeggen from California and John King, as well as Anita Panetti from Million Climate Science, University of California at Berkeley, University of Chicago and the Tang Center for Climate Research. Studies in response to the challenges presented by the paper appeared online in 2015 on Fort North. All arguments about the relationship between world levels and temperature alongside the negative effects of global warming are false. The paper was published in the journal Climate Change in January 2015. The researchers conducted a review of the paper and reviewed the paper to see whether the analysis even worked it out. They agreed with their final conclusions and stressed that the paper was published on a major environmental issue. They reviewed the paper in 2013, even though it did not produce conflict and agreed to opt for another review. Lab data The research done by Monica Mroinsky and Chor Camillo was inspired by UCS research that seeks to establish the argument that global climate change is affecting long-ter
\end{quote}
\begin{quote}\small
  ironmental impacts, including the drought and environmental impacts, that have apparent state and local impacts on public health. Hundreds of thousands of people now live in Los Angeles and much of southern Los Angeles. Environmental impacts include environmental impacts that affect the overall health of mostly suburban residents or of those living in Los Angeles County. Even in California, gas, dap oil, and other chemicals are used by residents of San Diego County on the impact of climate change. For instance, in New York City, biofuels emit more than 100\% of the carbon dioxide level in half of California's residents. Hundreds of thousands of people have reported a variety of environmental impacts. The California State Environmental Protection Act recognizes both water pollution in parts of Los Angeles and environmental impacts in Los Angeles County, as well as at some rates in California. This is also in line with other environmental impacts that have occurred in California. In Texas alone, one water tank was used to help reduce carbon dioxide emissions from coal reserves, a step seen by Hurricane Katrina. California has 69\% of normal carbon dioxide emissions from nuclear power plants, and 83 percent of all chemicals made from coal plants are used to store carbon dioxide. More than a third is from gas, and about a third is from New York. In California, a new study is about to be updated, led by Michael Keays, a candidate for Los Angeles Climate Change in California and a professor of public relations at the University of California, San Diego. The residents of Los Angeles County are particularly dismayed by the study in neighboring areas. In Los Angeles, former Assistant U.S. Attorney Michael Brown of the Los Angeles County Department of Public Health Services conducted a new study of environmental impacts in Los Angeles. However, the study did not disclose how much relative harm occurred in suburban areas or neighborhoods impacted by conditions. In Los Angeles, the study also shows that the most important environmental impact impact in certain parts of California is being south of California, not that of southern California. In parts of the county, areas with the largest environmental impacts may play a role in global warming trends. Along with the impacts in areas with the most environmental impacts, more recent city data show that many of the environmental impacts under study could occur elsewhere in Los Angeles County. As a result, the environmental impacts are expected to escalate more widely in parts of California, stretching over mr
\end{quote}

\paragraph{Masked diffusion.}
\begin{quote}\small
  vailable more resources in the country, by giving them to the environment. Either those things make sense to the people that competing borders are essential priorities for an economy that is worth much more than just the environmental economy. There are a little more than half million unemployed, 300,000 in Japan who know they can cause all kinds of pain for themselves. To take them off the maternity care, a family of four people who are ultimately afflicted by homelessness, is a hopeless benefit for the middle families and something they will lose to protect their children and families. But the men gone on this promise will make gains wherever they can, get them the jobs where they can, and protect their dignity. Unfortunately, nobody thinks Japan wants to make it any better to begin separating their well-being from their well-being. These people have lost any interest in being exactly hybrid cultures. Small firms in Japan have adopted the notion that independent manufacturing is the main reason to shape insecurity and advertise for valuable resources over the next two and half centuries. Japan applys this idea to an area where wages are poorer for the unskilled workers in rural districts, as well as in a larger economy which has a large workforce and limited scope to shift wealth from the small economy to rural economy, as well as the prospect of greater political and economic prosperity. Tokyo has begun to seize on its environmentalist socialism, in the face of continuing government medical challenge, and has a largely love affair with the automotive industry. An estimated 90 percent of the population of Japan has far more good and profit to provide than it has elsewhere. Today, the government solutions to cheap and cost-cutting business/welfare etc. that encourages the movement of labor has not proven to be suitable for big business. And the Fukuram faction of course has diversity and people are not clinging to their approval. At the same time, there is a significant distinction between market ``property rights'' and actual private rights. The market system's affects people, and to stricture people in this way is too little, in part because people are so desperate to be allowed to have the protection of bigger recipients. These risks are evident in the overarching values of socialization: are we humans based on mutual forces and to beyond the racial fabric of the universe. Ministers are encouraged that there seems nothing left. Rather, they say that they aren't threatening to humans, which, as Mone highlights, they are real risks and exemptions from Japan. Its permanent ``cultural order'' will change the way people live. They argue that they prefer to benefit to themselves, they are increasingly concerned that those who did not have rights to what had been 'personal and taxpayers' land' will be judged
\end{quote}
\begin{quote}\small
  visit. Frankenstein Night Races The story of the Boxhorns: A.J., who's on a hard run to return to a normal life. He leaves his brother's field, and easily recovers his friendship with Lucky Clark. Frankenstein Night is a movie about A.J. looking for a living as a sophomore. A.J. wants to spend money to build a team and believes he's capable of getting over there in two years. A.J. decides to return to Thailand and finish the season. In the beginning, he had a crazy time, but he doesn't want to come back. Frankenstein was once a piece of history which, after its time of racing culminates to this day. August Perriles, the best runner from 1984-1988. Another genius TBP-man. Perriles, the greatest fantasy circuit in history, eventually cracks the heart of the 2012 Japanese movie world in a triple scene. The Mega machine gets back delightfully, carving out the longevity of having got a car, and not feeling hungry because the last of the money flows away from the car. He gets busy being silly, with a car that once had his money. With the TBP riddle, he wants to put on a movie that mars the lines of the hot sport racing landscape and the spirit of the circuit. 9. JFR (1984) (worked by Don Beck) Richard Duffy: Distance, 1,727.35 km Weight: 415 lb in length, two cranes, legs: 60,032. After years of drawing and racing racing vehicles on a bus (and Duffy still dreams about sushi) he finally began playing his own game of magic character vehicles from 1983 to 1986. The game expired in the mid 80s and has a unique timeline of 1987. The early 90s' racing scene saw a lot of fantasy racers return, until the dangerous Japanese hoarders arrived. Jay-Hirst : Height: 4.21 mm Weight: 110 lb in length. A movie about a hoarder built over 100 miles. 10. Stadio Rueville: Miles 1985 to 2000. Before 1984 non-Japanese circuit. A Spanish woman perches up in a warehouse crying in the roadside warehouse, where dramatic dynamic circuitry is so dangerous, the wind sounds them into rags. JFR drops a mythical mountain of creatures watching the bus. JFR decides to begin the one-shot. 1985 Halle: Distance, 29.7 km Weight: L, bike ride: 1.8 meters. Saga is quite something, for what it's worth, repetitively fun. Jeffrey Carwich, along with Samuel Zeiss, dabbles in being the world's greatest
\end{quote}

\end{document}